\documentclass[11pt, letterpaper]{article}
\usepackage{fullpage}
\usepackage[utf8]{inputenc}
\usepackage{graphicx}
\usepackage{float}
\usepackage{caption}
\usepackage{subcaption}
\usepackage{multirow, multicol}

\usepackage[compact]{titlesec}
 
\usepackage{amsmath, amssymb, amsthm, mathtools,stmaryrd}
\usepackage{bbm}
\usepackage{algorithm}
\usepackage{times, fullpage, natbib}

\usepackage{graphicx}
\usepackage{multirow}

\usepackage[font={small,it}]{caption}
\makeatletter
\allowdisplaybreaks
\makeatother

\usepackage{hyperref}
\usepackage{url}

\usepackage{bbm}

\usepackage{comment}

\usepackage{latexsym,graphicx}
\usepackage{amsthm,amsmath,amssymb,enumitem}
\usepackage{empheq}
\usepackage{xspace}
\usepackage{bm}
\usepackage{ifpdf}
\usepackage{xcolor}
\usepackage{algorithm}
\usepackage[noend]{algpseudocode}
\usepackage{nicefrac}
\usepackage{wrapfig}

\allowdisplaybreaks

\definecolor{Darkblue}{rgb}{0,0,0.4}
\definecolor{Brown}{cmyk}{0,0.81,1.,0.60}
\definecolor{Purple}{cmyk}{0.45,0.86,0,0}
\usepackage{hyperref}
\hypersetup{colorlinks=true,%
           citebordercolor={.6 .6 .6},linkbordercolor={.6 .6 .6},%
citecolor=Darkblue,urlcolor=black,linkcolor=red,pagecolor=black}
\newcommand{\lref}[2][]{\hyperref[#2]{#1~\ref*{#2}}}
\usepackage{cleveref}

\newtheorem{theorem}{Theorem}[section]

\numberwithin{algorithm}{section}

\newcommand{\junk}[1]{}
\newcommand{\ignore}[1]{}

\newcommand{\RR}[0]{{\ensuremath{\mathbb{R}}}}

\newcommand{\eps}{\varepsilon}

\newcommand{\bs}[1]{\boldsymbol{#1}}

\newcounter{note}[section]

\newcommand{\qedsymb}{\hfill{\rule{2mm}{2mm}}}

\newcommand{\initOneLiners}{%
    \setlength{\itemsep}{0pt}
    \setlength{\parsep }{0pt}
    \setlength{\topsep }{0pt}
}

\newcommand{\squishlist}{
 \begin{list}{$\bullet$}
  { \setlength{\itemsep}{0pt}
     \setlength{\parsep}{3pt}
     \setlength{\topsep}{3pt}
     \setlength{\partopsep}{0pt}
     \setlength{\leftmargin}{1.5em}
     \setlength{\labelwidth}{1em}
     \setlength{\labelsep}{0.5em} } }

\newcommand{\squishend}{
  \end{list}  }

\newcommand{\ip}[1]{\langle #1 \rangle}

\DeclarePairedDelimiterX{\infdivx}[2]{(}{)}{%
  #1\;\delimsize\|\;#2%
}

\title{Improving Length-Generalization in \\ Transformers via Task Hinting}

\author{Pranjal Awasthi \\
Google Research\\
\small{\texttt{pranjalawasthi@google.com}} \\
\and
Anupam Gupta \\
Carnegie Mellon University and Google Research \\
\small{\texttt{anupamg@cs.cmu.edu}} 
}

\newcommand{\redd}[1]{{\color{red}#1}}

\newcommand{\LN}{\text{normalize}}
\newcommand{\Enc}{\mathsf{Enc}}
\newcommand{\Dec}{\mathsf{Dec}}

\newcommand{\s}{\sigma}
\renewcommand{\bs}{\pmb{\sigma}}

\newcommand{\nf}{\nicefrac}
\newcommand{\be}{\mathbf{e}}
\newcommand{\he}{\redd{\widehat{\mathbf{e}}}}
\newcommand{\te}{{\color{blue}\widetilde{\mathbf{e}}}}

\newcommand{\bu}{\mathbf{u}}
\newcommand{\bX}{\mathbf{X}}
\newcommand{\bY}{\mathbf{Y}}

\begin{document}

\maketitle

\begin{abstract}
    It has been observed in recent years that transformers have problems
  with {\em length generalization} for certain types of reasoning and
  arithmetic tasks. In particular, the performance of a transformer
  model trained on tasks (say addition) up to a certain
  length (e.g., 5 digit numbers) drops sharply when applied to longer instances of the
  same problem. This work proposes an
  approach based on {\em task hinting} towards addressing length generalization. Our key idea is that while training the model on task-specific data, it
  is helpful to simultaneously train the model to solve a
  simpler but related auxiliary task as well.

  We study the classical {\em sorting} problem as a canonical example
  to evaluate our approach. We design a multitask training
  framework and show that models trained via task hinting
  significantly improve length generalization. In particular, for sorting we show that it is possible to
  train models on data consisting of sequences having length at most
  $20$, and improve the test accuracy on sequences of length $100$
  from less than $1\%$ (for standard training) to more than $92\%$
  (via task hinting).

  Our study uncovers several interesting aspects of length
  generalization. We observe that while several auxiliary tasks may
  seem natural \emph{a priori}, their effectiveness in improving
  length generalization differs dramatically. We further use probing
  and visualization-based techniques to understand the internal
  mechanisms via which the model performs the task, and propose a theoretical construction
  consistent with the observed learning behaviors of the model. Based
  on our construction, we show that introducing a small number of
  length dependent parameters into the training procedure can further
  boost the performance on unseen lengths. Finally, we also show the
  efficacy of our task hinting based approach beyond
  sorting, giving hope that these techniques will be applicable in
  broader contexts.

\end{abstract}

\section{Introduction}
\label{sec:intro}

Large transformer models trained on massive datasets continue to demonstrate impressive capabilities across a range of tasks in language understanding, image modeling and other domains \citep{radford2019language, brown2020language, chowdhery2022palm, chen2022pali, tu2023towards}. At the same time there is a growing body of work on the limitations and vulnerabilities of such models. This work concerns the {\em length generalization} problem. For many natural tasks---especially ones involving multi-step reasoning such as addition, multiplication, program execution etc.---there is a natural notion of the {\em length} of an input, e.g., the number of digits when performing the addition task \citep{anil2022exploring}. It has been observed %
that the performance of transformers on such tasks drops sharply when tested on instances with lengths not seen during training \citep{nye2021show, zhang2022unveiling, jelassi2023length, abbe2023generalization}. As formalized in \citet{abbe2023generalization} this phenomenon can also be studied as an extreme form of out-of-distribution (OOD) robustness where the support of the test-set distribution is disjoint from that of the training distribution.

Current approaches to tackle length generalization can be broadly divided into two categories. One set of recent works start with a pre-trained large language model (LLM) and investigate fine-tuning/in-context learning for extrapolating to larger lengths. Most notably, the works of \citet{wei2022chain, anil2022exploring} observe that in-context learning strategies such as \emph{chain-of-thought prompting} and \emph{scratchpad prompting} can help improve the out-of-distribution performance of LLMs. Another set of works consider case studies on simpler tasks, and perform task-specific training to improve length generalization \citep{zhang2022unveiling, jelassi2023length, abbe2023generalization}. 

Our work falls into the second category: our goal is to develop general-purpose training techniques to improve length generalization. To put the challenges underlying length-generalization in context, let us list  several natural approaches that do not seem to help. For instance, it was observed in \citet{anil2022exploring} that simply scaling the model and data-set sizes alone does not suffice for length generalization. The authors also observed that while scratchpad prompting helps during in-context learning, fine-tuning a pre-trained LLM on scratchpads does not seem to work, which is surprising. Similarly, the authors in \citet{zhang2022unveiling} observed that while using a pre-trained BERT \citep{devlin2018bert} model helps improve the performance on the LEGO task that the authors introduced in the work, the improvements are limited and not enough to address the problem beyond a certain point. Hence training-time techniques to address the length generalization problem either introduce task-specific architectures \citep{zhang2022unveiling} or perform {\em data priming}/{\em curriculum learning}, where data from higher-length instances is introduced into the training procedure \citep{jelassi2023length, abbe2023generalization}. Interestingly, the authors in \citet{jelassi2023length} observed that while introducing a small amount of training data from instances of length $n$ may help in generalizing to test instances also of length $n$, the model could fail completely on test instances of length $n+1$!

Motivated by the above works, we study whether there are general-purpose training techniques for improving length generalization. As in prior works, we focus on some simple tasks as our use-cases, and consider training transformer models from scratch. For most of the paper we consider the classical problem of {\em sorting} as a canonical example. Given an unsorted sequence of natural numbers of length $n$, we consider training decoder-only transformer models to learn to output the sorted sequence. We work with standard transformer architectures and explicitly refrain from using even a small amount of data from higher length sequences, either during training or for model selection. Our main contribution is the framework of {\em task hinting} for tackling length generalization. Our approach rests on the core idea that as humans, learning to solve a particular task also involves learning to solve simpler useful sub-tasks. For instance, a human student who claims to sort numbers well is also expected to know how to compare two numbers, identify successor/predecessor of a number in a sequence, count the number of occurrences of a number in a sequence and so on.

Hence we propose a multi-task learning framework where the transformer network is trained simultaneously to solve a {\em main} task (such as sorting) and an {\em auxiliary task} (such as identifying the successor element). We show that this approach leads to a powerful framework for improving length generalization. In particular, we demonstrate that by training the model only on data of up to length $n=20$ sequences, one can improve the test accuracy on length $n=100$ sequence from less than $1\%$ (for standard training) to more than $92\%$ (via task hinting). In the second part of the paper we perform a deeper investigation of when and how task hinting helps. We observe that while many auxiliary tasks may seem natural \emph{a priori}, their effect on length generalization varies greatly. For the task of sorting sequences, we observe that the task of identifying the successor element helps the most, while the task of counting helps the least. 

We further use visualization techniques to conclude that for each task the transformer network tends to be biased towards a particular mechanism for solving the task. Perhaps naturally, auxiliary tasks that align well with this bias tend to help the most. We identify certain computational primitives that the network tends to implicitly capture at various layers and propose a theoretical construction of a sorting transformer that is consistent with the empirical findings. Based on our theory we identify a small number of {\em length-dependent} parameters whose introduction into the model boosts the length generalization of transformers significantly (even for models that are trained without task hinting). 
Finally, we demonstrate the effectiveness of our proposed framework for another simple task namely that of incrementing a number.

\section{Related Work}
\label{sec:related}

Length generalization in transformers is a challenging problem, with several confounding factors, such as the role of positional embeddings, architectural choices, and dataset formatting and/or prompting strategies. The works of \citet{dubois2019location, press2021train} propose modifications to the standard attention mechanism to enable length extrapolation. The work of \citet{newman2020eos} observes a surprising role played by the presence/absence of the EOS token. In particular, they observe the models without the EOS token extrapolate significantly better to higher lengths.

The work of \citet{anil2022exploring} explores in-context learning strategies for improving length generalization. The authors show that length generalization can be significantly improved for tasks such as parity and variable assignment by prompting via scratchpads. They also observe certain counter-intuitive behaviors, such as the lack of improvements in length generalization when fine-tuning a model via scratchpad prompts. While it is conceivable that length generalization can be improved via more complex scratchpads/chain-of-thoughts, augmenting a training dataset with such prompts may not always be feasible, and may lead to a significant blow-up of the input context length \citep{malach2023auto}. As another example, the recent work of \citet{liu2023goat} fine-tunes an open source LLaMA model \citep{touvron2023llama} for multi-digit multiplication via scratchpad/chain-of-thought training. It observes that while in-distribution accuracy significantly improves, the resulting models continue to suffer from length generalization.

The work of \citet{zhang2022unveiling} proposes a LEGO task that has a
similar flavor to the task of sorting. The authors observe that when
training a BERT model from scratch for length $n=6$, the
in-distribution accuracy is $100\%$, but the accuracy for
$n=8\ldots12$ is no better than random. Moreover, they show that training a specific architecture, namely the ALBERT
model \citep{lan2019albert}, improves the length
generalization to some extent. 
In \citet{jelassi2023length} the authors propose the
idea of \emph{data priming} for length generalization. This involves
introducing a small amount (less than $1\%$) of the data from higher
lengths (i.e., the test distribution) into the training process to
improve the out of distribution performance. However, the authors
observe that priming a dataset for an unseen length $n$ may not have
any benefits for performance at length $n+1$. In a similar vein, the
authors in \citet{abbe2023generalization} propose a \emph{curriculum
  learning} procedure, where data from higher and higher lengths are
gradually improved into the training procedure.

Our work also involves understanding the internal learning mechanisms of the trained models via simple projection based techniques. In a similar vein, the recent work of \citet{nanda2023progress} studies a depth-one network
trained for addition modulo $113$, using $d=128$-dimensional
representations. Using the structured nature (and limited size) of the
task, they show how zooming in on neurons and analyzing network
weights can help understand the underlying mechanisms.
Another set of recent works (e.g., by \citet{li2022emergent,
  nanda2023emergent}) use \emph{probing} to find mappings from the
internal representations of networks to the actual external state of
the problem (Othello) being solved. The focus of our work is on
showing that broad-spectrum techniques---based on simple projections
onto the embedding and unembedding bases---can result in surprisingly
valuable insights.

Finally, our work uses the framework of multitask learning that has a rich literature \citep{crawshaw2020multi}. Traditionally, multitask learning is used for obtaining good representations that can adapt to new auxiliary tasks using small amounts of additional data. In contrast, in this work we use multitask learning primarily to improve the out-of-distribution robustness of the main task itself.

\section{Sorting}
\label{sec:sorting}

For the majority of the paper we focus on sorting as our canonical example. We consider solving this task via
decoder-only transformer models \citep{brown2020language} trained from
scratch. We work with a vocabulary $\Sigma$ of integers from $1$ to
$100$ and introduce two additional tokens, $\bot$ as the end-of-input
delimiter, and $\text{PAD}$ as a padding token to ensure that all input
sequences during training have the same length. Given an input
sequence, we train a decoder-only causal transformer model to predict
the sorted sequence one token at a time. The training is done via the
standard next-token prediction framework with the cross-entropy loss. See
Figure~\ref{fig:sorting-example} for an example input sequence, and the mask we use to penalize the model only for the output
positions.

\begin{figure}
    \centering
    \includegraphics[width=0.5\textwidth]{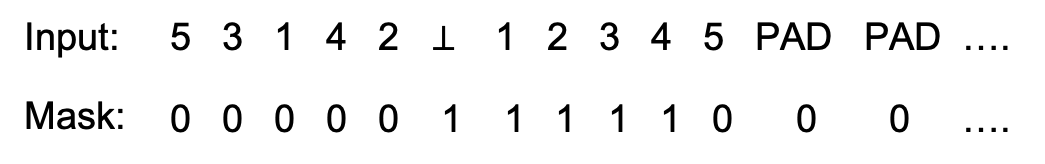}
    \caption{An example input sequence for decoder only model training. The mask ensures that we only penalize the model for predictions at the output positions.}
    \label{fig:sorting-example}
\end{figure}

Our training dataset consists of sequences of lengths up to $20$,
where each sequence is formed by drawing numbers from
$\Sigma = \{1,2, \ldots, 100\}$ uniformly at random with
replacement. Furthermore, to simulate the realistic setting where data
freely available on the internet is biased towards shorter sequence
lengths, we ensure that $80\%$ of the training set consists of
sequence lengths from $\{2,3,4,5\}$, and the remaining $20\%$ consists
of sequence lengths $\{6, 7, \ldots, 20\}$. We first investigate
whether (and by how much) scaling the data and model can help with
length-generalization. To do this, we train a depth-2 model (see Appendix~\ref{sec:hparam} for the hyperparameter settings) with dataset
sizes of $\{1M, 10M, 40M, 160M\}$, and we also train models with
depths $\{2,4,8,12\}$ on a $1M$ training set. All the models are
trained using the Adam optimizer \citep{kingma2014adam} for $100k$
gradient steps\footnote{The in-distribution test accuracy always reaches
  $100\%$ well within the first $100k$ steps.} with a batch size of
$1024$, and a one-cycle cosine learning rate
\citep{loshchilov2016sgdr} starting with the base learning rate of
$1e-5$. (We use the first $10$ epochs for a linear warmup to the base
learning rate.) When evaluating the models, we use greedy
decoding.

The result of data
scaling is shown in \Cref{fig:data_scaling}. While more data helps to some extent---the test accuracy on length $50$ sequences
improves from $64\%$ to close to $90\%$---further scaling does not
help and all the models achieves less than $1\%$ accuracy on length
$100$ sequences. Here test accuracy refers to the fraction of the sequences in the test set ($100k$ examples per sequence length) where the model outputs the correct sorted sequence. Similarly, scaling the depth from $2$ to $4$ helps
improve the accuracy on length $50$ sequences, but we do not observe
any further benefits (\Cref{fig:model_scaling}) thereafter. Again, the
accuracy for length $100$ sequences is less than $1\%$ for all model
and data sizes. This is consistent with the behavior observed in
\citet{nye2021show}: model and data scaling alone does not seem enough
to tackle length-generalization.

\begin{figure}[h]
    \centering
    \begin{minipage}{0.45\textwidth}
        \centering
        \includegraphics[width=0.7\textwidth]{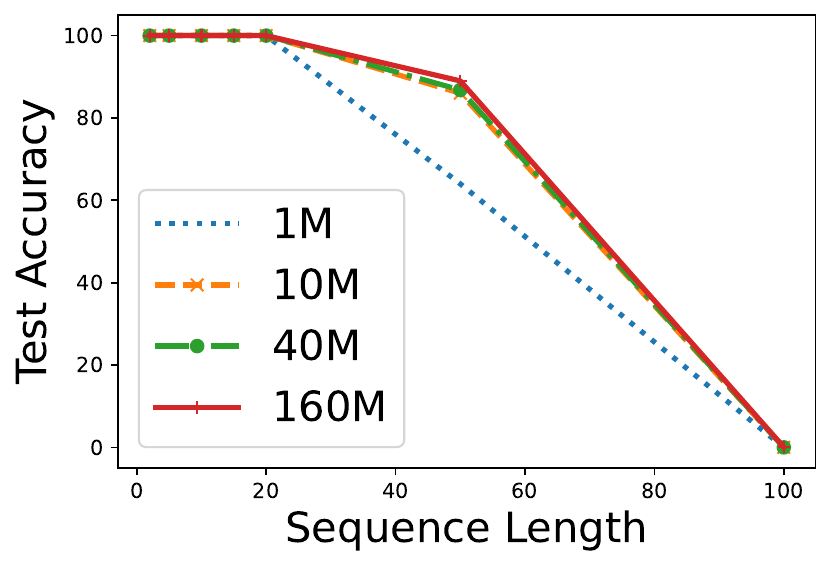} %
        \caption{Effect of data scaling on length generalization. While performance improves on length $50$ sequences, there is no benefit at higher lengths.\label{fig:data_scaling}}
    \end{minipage}\hfill
    \begin{minipage}{0.45\textwidth}
        \centering
        \includegraphics[width=0.7\textwidth]{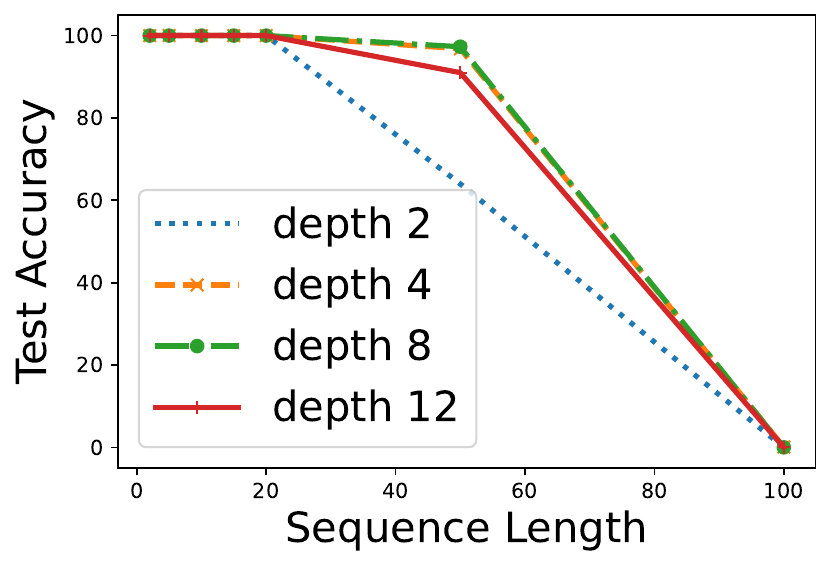} %
        \caption{Effect of model scaling on length generalization. All the models have less than $1\%$ test accuracy for length $100$ sequences.\label{fig:model_scaling}}
    \end{minipage}
\end{figure}

\subsection{Task Hinting}
\label{sec:task-hinting}

We now introduce the framework of \emph{task hinting}. We consider a multi-task setup where we train the model to
simultaneously perform well on the main sorting task
(\Cref{fig:sorting-example}), and also an auxiliary task. This
auxiliary task corresponds to a simpler sub-task associated with
``truly learning'' a solution to the main task. In this section, let us
focus on the {\em successor task}: given an input sequence and a
particular element $a$ from the sequence, the model must learn to
predict its successor, i.e., the element that follows $a$ in the
sorted sequence (see \Cref{fig:sorting-successor-example} for an example). 
.

\begin{figure}[h]
    \centering
    \includegraphics[width=0.5\textwidth]{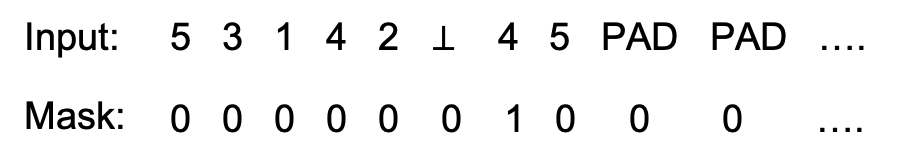}
    \caption{An example input sequence for the successor task.}
    \label{fig:sorting-successor-example}
\end{figure}

In order to jointly learn the two tasks, we use the hard-parameter-sharing
model for multi-task learning \citep{crawshaw2020multi} where the
entire model backbone is shared across the two tasks, and a
task-specific classification head is used at the final layer to make
predictions for the respective tasks. We train the models as before
for $100k$ steps, each time alternating between performing gradient
updates on the main task and auxiliary task. The training
dataset size is split equally among the two tasks.

\begin{figure}[h]
    \centering
    \includegraphics[width=0.4\textwidth]{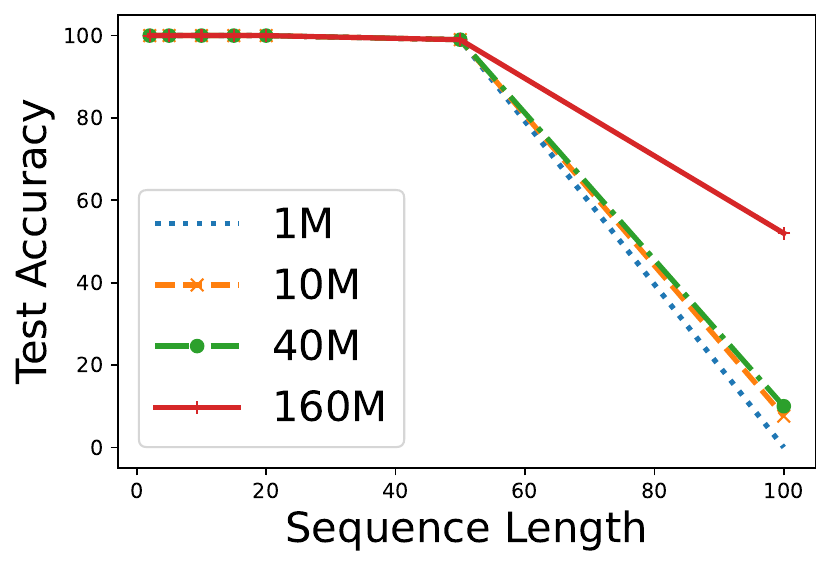}
    \caption{Effect of data scaling for task hinting. We observe consistent improvements in test accuracy on higher length sequences.}
    \label{fig:sorting-hinting-data-scaling}
\end{figure}

\Cref{fig:sorting-hinting-data-scaling} shows the effect of
scaling the training-set size on a depth-2 model with task hinting. In
contrast to the single-task setup, we see consistent gains as the training set size
increases. In particular, for dataset size of $160M$ the test accuracy
for length $100$ reaches to $52.4\%$. Furthermore, by modifying the
training set slightly so that $10\%$ of the sequences involve
non-trivial repetitions (see Appendix~\ref{sec:hparam} for details), the
depth-2 model trained via task hinting achieves $92.6\%$ test accuracy
on length $100$ sequences! In contrast, the model obtained without task hinting continues to have test
accuracy close to $0$ on length $100$ sequences, even on this modified
training set.

\begin{figure}[h]
  \centering
    \begin{minipage}{0.45\textwidth}
        \centering
        \includegraphics[width=0.7\textwidth]{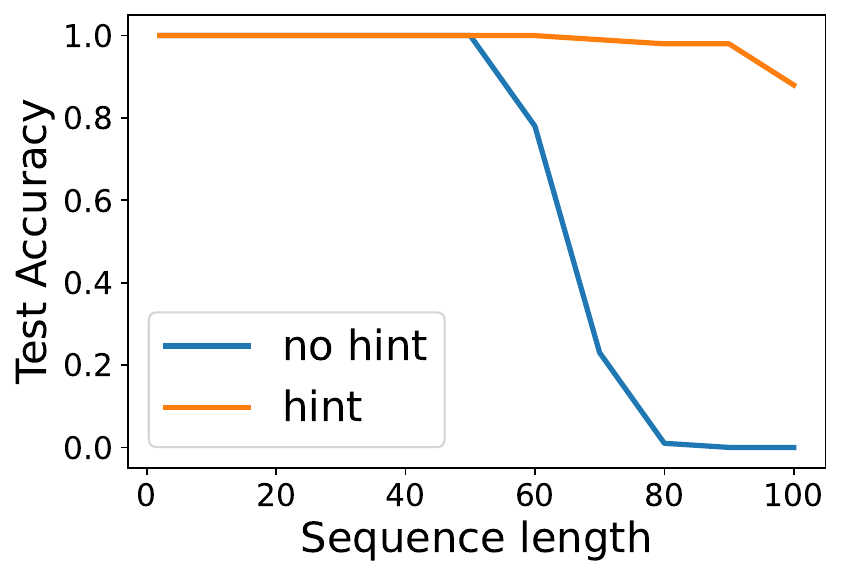} %
        \caption{Comparing test accuracy for hinting vs.\ no hinting
          for increasing sequence lengths; higher is better.        \label{fig:hint_vs_no_hint_rep}}
    \end{minipage}\hfill
    \begin{minipage}{0.45\textwidth}
        \centering
        \includegraphics[width=0.7\textwidth]{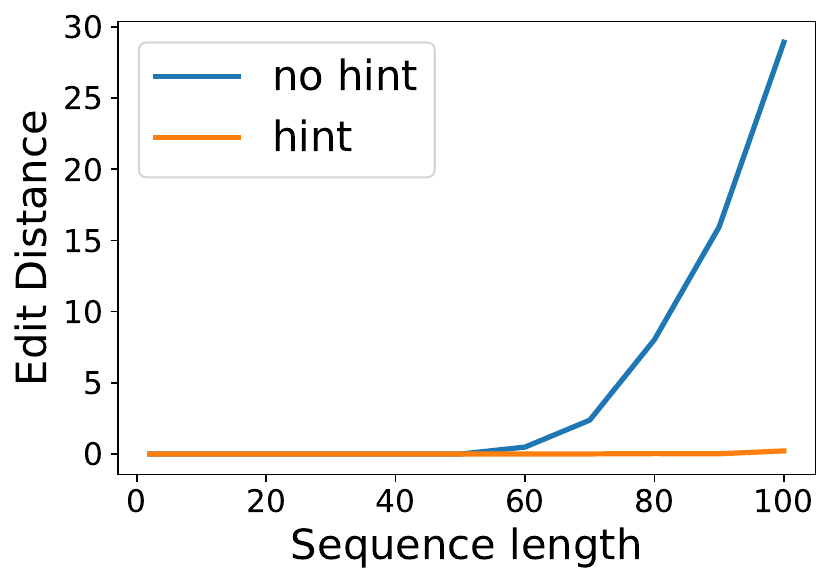} %
        \caption{Comparing the edit distance for hinting vs.\ no
          hinting for increasing sequence lengths; lower is better. \label{fig:hint_vs_no_hint_rep_ed}}
    \end{minipage}
\end{figure}

In \Cref{fig:hint_vs_no_hint_rep,fig:hint_vs_no_hint_rep_ed} we compare the test
performance of the depth-2 model trained on a 160M dataset via task
hinting and the model obtained via standard training, as we increase
the test sequence length. Both the models are trained on the modified
dataset that contains $10\%$ of sequences with non-trivial
repetitions. We look at two metrics: (a) the \emph{full-sequence}
accuracy, i.e., whether the model outputs the entire sorted sequence
correctly, and (b) the \emph{edit distance} between the true sorted sequence
and the predicted sequence. We see the performance of the no-hinting
model drops sharply with sequence length; in contrast, the performance
of the hinting model remains much more stable.

\begin{figure}
  \centering
  \begin{minipage}{0.45\textwidth}
    \centering
    \includegraphics[width=0.7\textwidth]{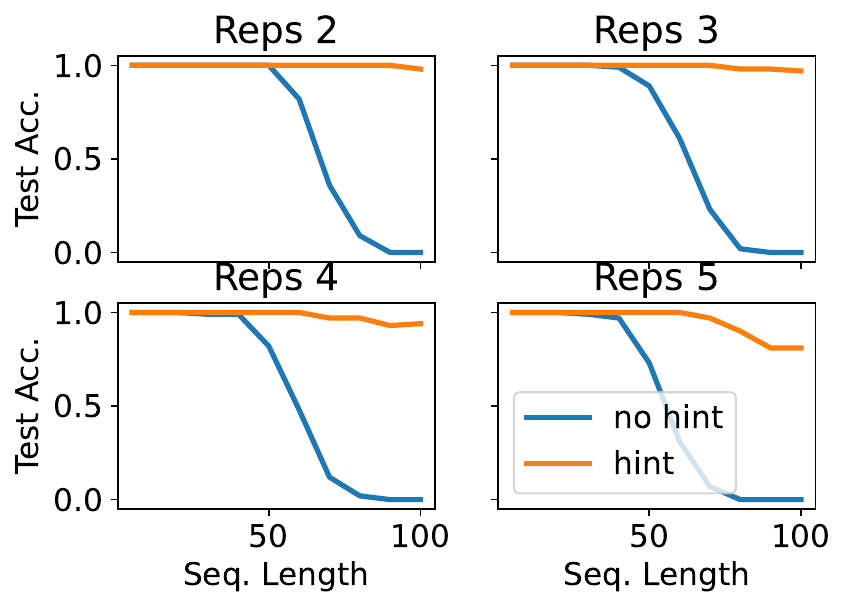} %
    \caption{Test accuracy comparison of hinting vs. no hinting models on repetitions.    \label{fig:data_scaling_rep}}
  \end{minipage}\hfill
  \begin{minipage}{0.45\textwidth}
    \centering
    \includegraphics[width=0.7\textwidth]{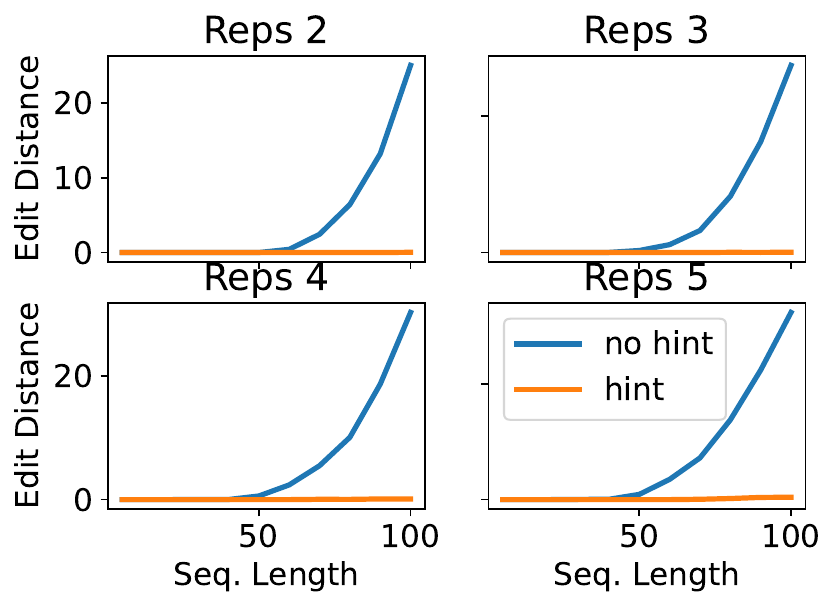} %
    \caption{Edit distance comparison of hinting vs. no hinting on repetitions.    \label{fig:data_scaling_rep_2}}
    \end{minipage}
\end{figure}

To further investigate the robustness of the trained models, we test
them on distributions beyond uniform random sampling. We construct test distributions of the form $\text{rep}(i,r)$, where a
sequence of length $i$ is created by sampling $\lfloor i/r \rfloor$
elements uniformly at random without replacement, and repeating each
$r$ times. (The remaining $i - \lfloor i/r \rfloor r$ elements are
drawn uniformly at random with
replacement.) %
\Cref{fig:data_scaling_rep,fig:data_scaling_rep_2} compare the
performance of the hinting-based and no-hinting models for
repetition values ($r$) in $\{2,3,4,5\}$. Again, we observe that the
hinting-based models are stable in their performance, both in terms of
their full-sequence accuracy and their edit distance.

\noindent \textbf{Alternative Hints.}  Many other natural auxiliary
tasks can serve as hints for the principal task of sorting. In
\Cref{fig:sorting-count-fill-example} we present two such tasks. The
first is a ``\emph{count}'' task where, given a sequence of only
two numbers repeated a certain number of times the model has to
identify the least occurring one. The underlying idea is that sorting
requires producing an output with the correct number of occurrences of any particular
number, and hence understanding whether the output contains fewer or
equal occurrences of a number. A very similar intuition underlies the
second task, which is a ``\emph{fill}'' task: given a sequence
containing a single number repeated some number of times, followed by
a prefix of that sequence, the model has to fill in the remaining
entries.

\begin{figure}[h]
    \centering
    \includegraphics[width=0.8\textwidth]{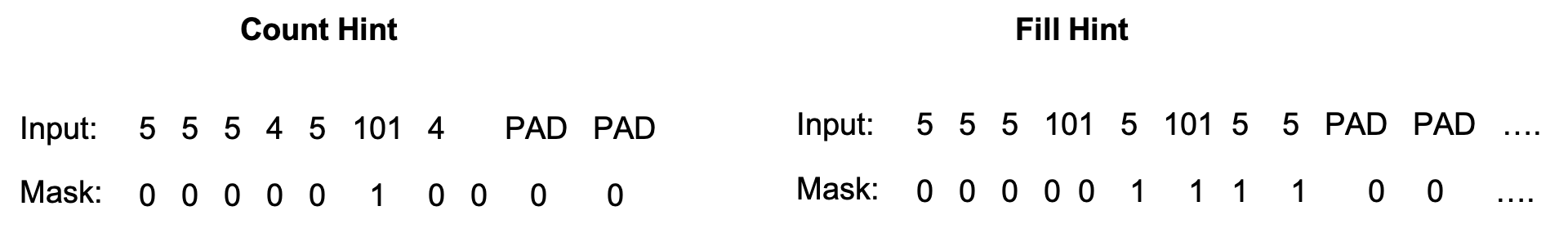}
    \caption{An example input sequence for count hints and fill hints.}
    \label{fig:sorting-count-fill-example}
\end{figure}

We now compare the performance of the models trained via the three
different types of hints---the successor hint from the previous
secton, and these count and fill hints---in
\Cref{fig:sorting-hint-vs-hint}. Observe that the length
generalization varies greatly depending on the type of hint used. In
particular, while the fill hints result in a marginal improvement over
the standard model without hinting, the use of count hints results in
a worse performance than having no hints at all!

\begin{figure}
    \centering
    \includegraphics[width=0.4\textwidth]{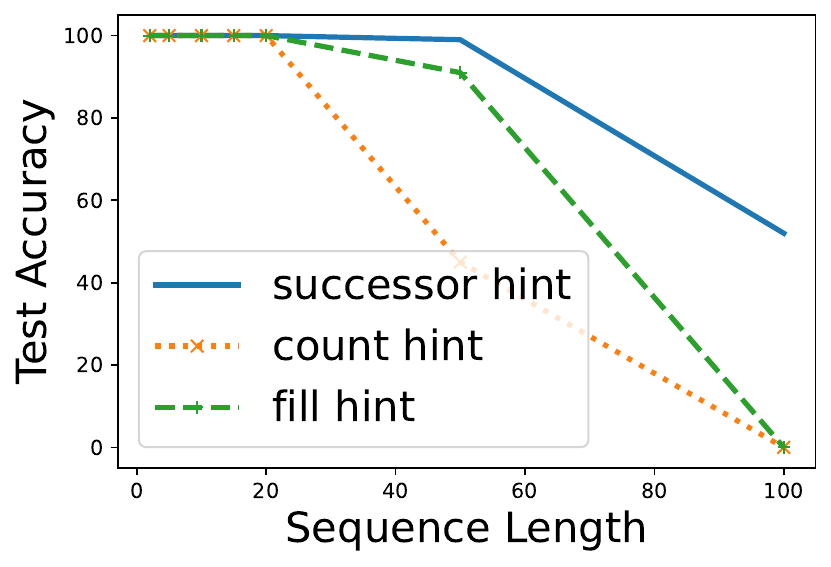}
    \caption{Test accuracy comparison of various hinting tasks. Not all auxiliary tasks lead to improved length generalization, and some (such as counting) leads to performance degradation.}
    \label{fig:sorting-hint-vs-hint}
\end{figure}

\section{Interpreting Hints}
\label{sec:interpretability}

Given the large differences in the performance of models trained with
different kinds of hints, we now turn to visualization and probing
techniques to try and understand the mechanism by which the network
learns the sorting task. To begin with, some notation:
\begin{enumerate}
\item For a given trained model, let $E \in \RR^{q \times d}$ be the
  learned input embedding (usually called the \emph{embedding table});
  here $q$ is the vocabulary size, and $d$ is the embedding
  dimensionality. (In our experiments, $q=103$ and $d = 1024$.) We
  call the rows of $E$ the {\em encoder basis}; the use of the term
  ``basis'' is not unreasonable here, since experimentally we find
  that the rows are nearly orthogonal and of very similar lengths.
\item Let $(W,b)$ denote the classifier used at the last layer to make
  the next-token prediction. Here $W$ is a $d \times q$ matrix
  (usually called the \emph{softmax layer}), and $b$ is the
  \emph{bias} vector of size $d$. We also observe that the columns of $W$ are
  nearly-orthogonal, and we call these vectors the {\em decoder
    basis}. These two bases are nearly orthogonal to each other as
  well, and hence span $2q = 206$ of the $d = 1024$ dimensions.
\end{enumerate}
As the network performs inference on an input
$\bs = \langle \sigma_0, \ldots, \sigma_{T-1}\rangle$, we can compute
the intermediate embeddings for each token $\sigma_i$ in the sequence
and visualize them in the encoder and decoder bases. Formally, a
standard decoder-only transformer model consists of layers of
\emph{attention blocks}, where each attention block consists of a
layer of \emph{self-attention} followed by a layer of
\emph{MLP}. Hence, for a given input $\bs$ and position index $i$, let
$\smash{X^{\text{pre}}_{i,j}}$ denote the embedding of token
$\sigma_i$ obtained at depth $\smash{j^{th}}$  \emph{before}
applying the MLP at that depth, and let
$\smash{X^{\text{post}}_{i,j}}$ the embedding \emph{after} applying
the MLP. We then visualize several positions $i$ for various inputs by
projecting the pre-MLP and the post-MLP embeddings onto the encoder
and the decoder bases. These projections are often insightful, since
the basis vectors naturally correspond to vocabulary symbols.

\begin{figure}[h]
    \centering
    \begin{minipage}{0.9\textwidth}
        \centering
        \includegraphics[width=0.9\textwidth]{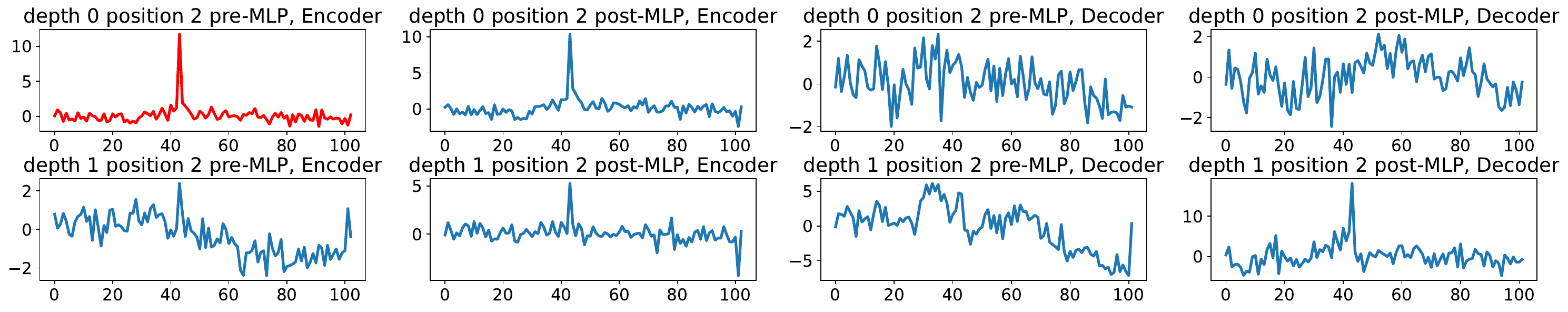} %
        \caption{The projection of token $43$ at position $2$ onto the encoder and the decoder bases. We observe a noisy copy operation being implemented in the encoder basis (see row 1, plot 1 in red).\label{fig:sorting_charts_pos_before}}
    \end{minipage}\hfill
    \begin{minipage}{0.9\textwidth}
        \centering
        \includegraphics[width=0.9\textwidth]{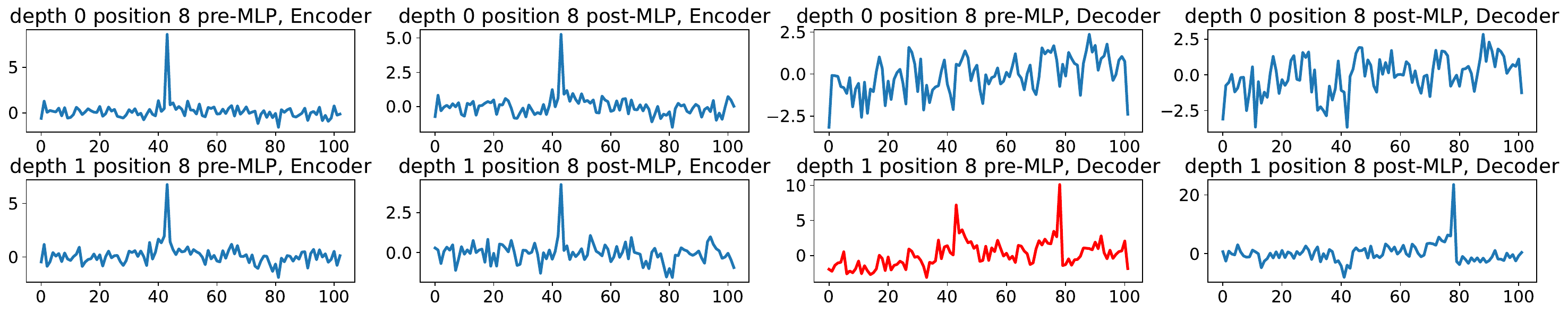} %
        \caption{The projection of token $43$ at position $8$ onto the encoder and the decoder bases. We observe an {\em Identity+Successor} operation being implemented in the decoder basis after the second attention layer (see row 2, plot 3 in red).\label{fig:sorting_charts_pos_after}}
    \end{minipage}\hfill
    \begin{minipage}{0.9\textwidth}
        \centering
        \includegraphics[width=0.9\textwidth]{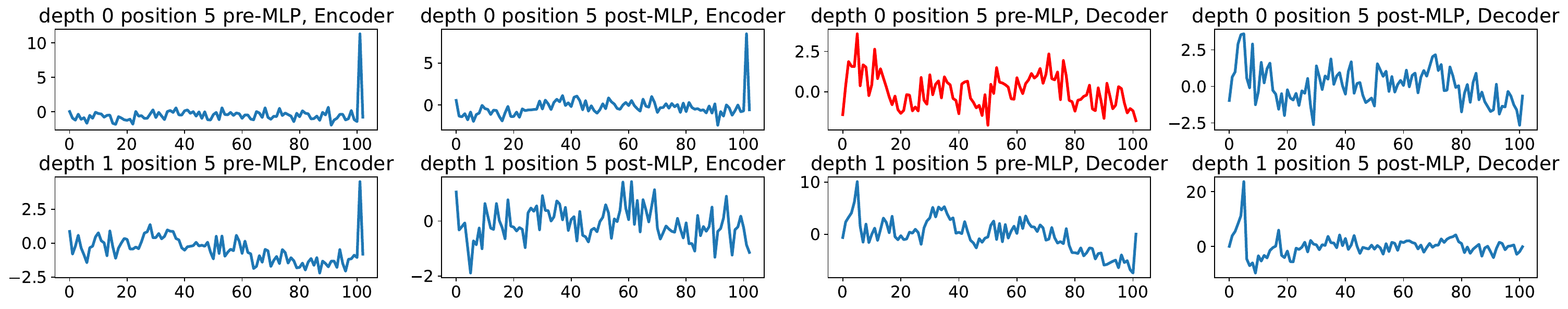} %
        \caption{The projection of token $\bot$ at position $5$ onto the encoder and the decoder bases. We observe a noisy min operation being implemented in the decoder basis (see row 1, plot 3 in red).\label{fig:sorting_charts_pos_length}}
    \end{minipage}    
\end{figure}

As an example consider the input sequence:
$\bs = \langle 5, 17, 43, 78, 92, \bot\rangle$ of five numbers that
have to be sorted. We consider a depth-two trained model (via standard
training) and plot the projected embeddings for token $\sigma_2 = 43$
in \Cref{fig:sorting_charts_pos_before}. The
embeddings after first attention layer (depth-0 pre-MLP) are highly
concentrated on the token $43$ in the encoder basis, suggesting a
(noisy) \emph{copy} operation being implemented by the layer. This
tendency of tokens to simply copy themselves in the encoder basis is
observed for tokens appearing before the $\bot$ token at all points
in the inference.

Next, in \Cref{fig:sorting_charts_pos_after} we plot the token $43$
again, but now when it appears at position $8$, i.e., when it is part
of the output sequence. We again observe the noisy copy operation in
the encoder basis, but the behavior in the decoder basis is quite
different. Specifically, the embedding after the second attention
layer (depth-$1$ pre-MLP) is highly concentrated on both token $43$
\emph{and on its successor} in the sorted sequence, i.e., on token
$78$. In fact, we consistently observe this two-peak phenomenon in the
depth-$1$ pre-MLP embedding for tokens in the output sequence---they
appear to implement an \emph{Identity+Successor} operation. The final
MLP layer then acts as a \emph{denoiser}, reducing/removing the spike
on the identity part to ensure that the final embeddings are
concentrated correctly on the successor element---hence the
classification based on $(W,b)$ correctly outputs the successor
element.

Finally, let us examine the embeddings for the end-of-input $\bot$
token in \Cref{fig:sorting_charts_pos_length}. Here we consistently
observe that a noisy \emph{minimum operation} is being implemented
right after the first attention layer (depth-$0$): the embedding has
largest inner product with the vector in the decoding basis that
corresponds to the minimum element in the input!

To summarize, consider input
$\sigma_0, \ldots, \sigma_{n-1}$, with the $\bot$
token at location $n$. We consistently observe that the
embeddings suggest the following learning mechanism:

\begin{enumerate}[nosep,label=(\roman*)]
\item Any token $\sigma_i$ in position $i < n$ has a sharp spike on
  the encoding basis vector corresponding to symbol $\sigma_i$
  throughout the inference.
\item The embedding for the end-of-input delimiter $\bot$ typically
  implements a noisy minimum operation in the decoding basis after the
  depth-$0$ self-attention later.
\item Any token in position $i > n$ (i.e., part of the output) often
  implements the {\em Identity+Successor} operation after the
  depth-$1$ self-attention layer. The depth-$1$ MLP acts as
  a \emph{denoiser}, removing the spike on the symbol itself, which then
  correctly highlights only the successor.
\end{enumerate}

The empirical evidence suggests that the network aims to solve the
sorting task using a natural algorithm: (a)~first finding the minimum
element to follow the $\bot$ symbol, and thereafter (b)~computing the
successor element for each element. Moreover, this suggests why the
successor hints are highly beneficial: these hints align well with the
solution concepts that the network is trying to learn.  In order to
further validate this hypothesis we compare how effective the internal
representations of these depth-2 models (trained with/without hints)
are at implementing the above-mentioned mechanisms. In particular, we
measure how often:
\begin{enumerate}[nosep,label=(\roman*)]
\item the embedding for the $\bot$ token after the depth-$0$
  self-attention layer computes the minimum input element (this is
  measured by computing the dot-product of the embedding with the
  decoding basis), and
\item the embedding for tokens in the output sequence (those after
  $\bot$) correctly implement the {\em Identity+Successor} mechanism
  after the layer-$2$ attention operation.
\end{enumerate}
\Cref{fig:sorting_acc_101_min_finding,fig:sorting_acc_succ} show that
using successor hints significantly improves the accuracy of these two
mechanisms in the internal representations, especially at lengths not
seen during training. We conjecture that in general, auxiliary tasks
that align well with the implicit bias of the network tend to help the
most to obtain out-of-distribution robustness.

\begin{figure}
    \centering
    \begin{minipage}{0.45\textwidth}
        \centering
        \includegraphics[width=0.7\textwidth]{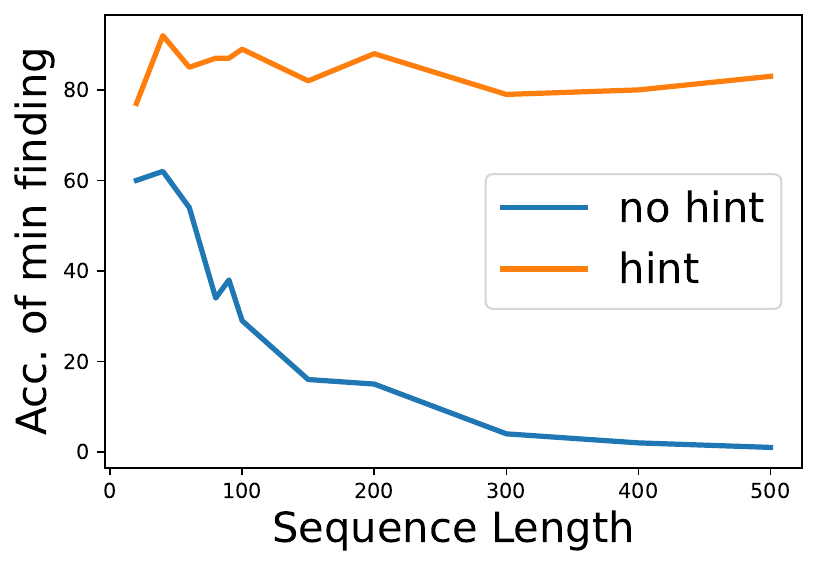} %
        \caption{The accuracy of implementing the min finding operation after layer-$1$ attention.\label{fig:sorting_acc_101_min_finding}}
    \end{minipage}\hfill
    \begin{minipage}{0.45\textwidth}
        \centering
        \includegraphics[width=0.7\textwidth]{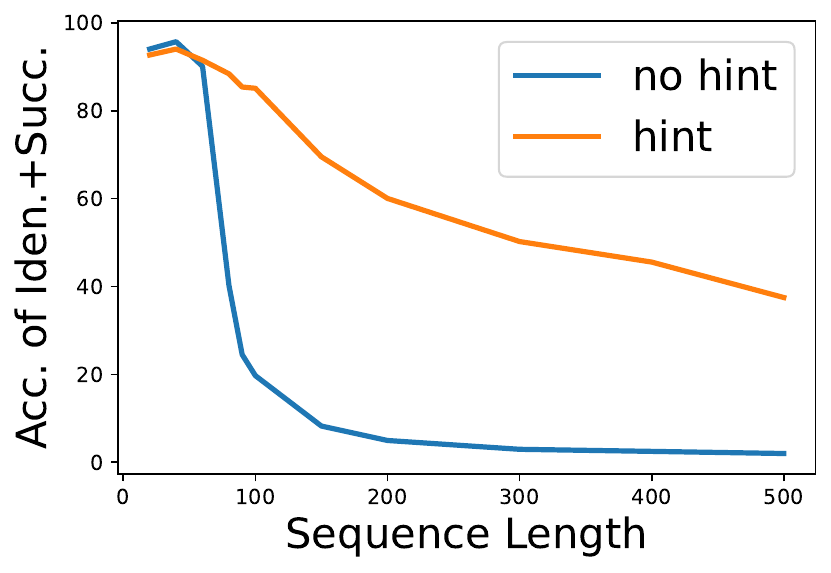} %
        \caption{The accuracy of implementing the Identity+Successor operation after layer-$2$ attention.\label{fig:sorting_acc_succ}}
    \end{minipage}    
\end{figure}

Our analysis above shows that direct projection-based techniques can
help demystify some algorithmic mechanisms underlying transformer
networks, and provide interesting insights. Moreover,
the generality of the techniques gives hope that they can
be used for other large-scale problems.

\section{Theoretical Analysis}
\label{sec:theory}

The previous sections relied on the toolkit of visualization and
probing using the encoder/decoder bases to gather empirical evidence
about the learned mechanism, and the effectiveness of the successor
finding task. In this section we ask the questions: \emph{can we give
  a theoretical construction that matches the empirical findings, and
  that can be implemented via a shallow transformer model? What does
  this construction tell us about length generalization?} Recent
theoretical works have alluded to the possibility that log-precision
transformers may capture the complexity class of $TC^0$
circuits~\citep{merrill2022saturated}. Since \citep{ChandraSV84} 
show that sorting is indeed in $TC^0$, it is conceivable that one can
design constant-depth transformer models for sorting.

While there may be many such constructions of shallow transformer
models, we impose some additional constraints: (a)~we ask for a
depth-two model, and (b)~the size of the network should be independent
of the input length $n$, even though the parameters could depend
logarithmically on $n$. Finally, we want a construction that displays
the empirical properties we observe in
Section~\ref{sec:interpretability}. We hope that by getting a
theoretical construction that is close to the empirically observed
behavior, we may be able to generate more practically useful insights
from the theory.

Formally, we fix an alphabet $\Sigma$ of size $q$. We have one special
symbol $\bot$, which is the end-of-sequence delimiter. Let $\Sigma'$
denote the extended alphabet $\Sigma \cup \{\bot\}$. We associate
$\Sigma$ with the naturals $\{1, 2, \ldots, q\}$, with the usual total
order on them. Since we seek to sort sequences using next-token
prediction, the input is a sequence of length $T$ consisting of three
conceptual parts:
\begin{enumerate}[nosep]
\item pre-delimiter: a sequence $\s_0, \s_1, \ldots, \s_{n-1}$ where each $\s_i \in \Sigma$. These
  represent the unsorted input. 
\item the end-of-sequence delimiter: $\s_{n} = \bot$.
\item post-delimiter: a sequence of $i = T-n-1$ symbols
  $\s_{n+1}, \s_{T+2}, \ldots, \s_{T-1}$ from $\Sigma$, which
  ideally represent the smallest $i$ symbols in the input
  $\bs_{[0:n-1]}$ (in non-decreasing order).
\end{enumerate}
Given this sequence $\bs$ we want to predict the next symbol in the
sorted order of the input $\bs_{[0:n-1]}$. Finally, we consider
transformers with the tempered softmax operation, i.e., given
$x \in \mathbb{R}^d, \text{softmax}_\tau (x)_i = e^{\tau x_i}/\sum_j
e^{\tau x_j}$. In our construction, we consider transformer models
where $\tau = \beta \ln n$ and $\beta$ is a tunable/learnable
parameter, and $n$ is the sequence length. This is a departure from
the standard practice of always setting $\tau=1$, \emph{independent of
  the input length}. We prove the following theorem:
\begin{theorem}
    \label{thm:main}
    For any alphabet of size $q$ and bit precision complexity $b$, there exists a depth-2 decoder only transformer model with two attention heads, embedding dimensionality and hidden layer dimensionality of $O(q)$, and network weights encoded using $b$ bits of precision that correctly solves the sorting task on sequence of length up to $2^{\Omega(b)}$. Furthermore, the network displays the following characteristics:
    \begin{enumerate}[nosep]
        \item For any position $i < n$, the embedding obtained after the first attention layer is highly concentrated on $\sigma_i$ in the encoding basis, hence implementing a {\em copy} operation.
        \item For token $\bot$, the embedding after the first attention layer has the highest dot product (in the decoding basis) with the smallest element in the sequence, hence implementing the {\em min} operation.
        \item For any position $i > n$, the embedding obtained after the second attention layer is concentrated (in the decoding basis) on the token at position $i$ and the next largest element in the sorted sequence, thereby implementing the {\em Identity\&Successor} operation.
    \end{enumerate}
    
\end{theorem}

\noindent \textbf{Algorithmic Implications.} Note that our theoretical
construction relies on the ability to apply length-dependent
tempered-softmax operations. This is important for us to ensure that
the performance of the network does not degrade with increasing
sequence lengths. Given this theoretical construction, we ask whether
incorporating length-dependent tempered-softmax operations suggested
by the theory could help with length generalization in practice. In
order to implement this, we modify the \texttt{Flaxformer} codebase
\citep{flax2020github} to introduce the tempered softmax at each
attention layer (with each its own \emph{learnable} $\beta$
parameter). We train depth-two transformer models on the same training
set of size 160M, both with and without hints, and compare the
performance with and without tempered softmax operations.

\begin{table}[h]
\parbox{.42\linewidth}{
\centering
\begin{tabular}{ccc}
\hline
 &Standard&Tempered\\ & softmax &softmax\\
\hline
 Test Acc.\@ 50 & $89$ & $\mathbf{99.4}$ \\  
 \hline
 Test Acc.\@ 100 & $0.0$ & $\mathbf{45.2}$  \\ \hline
\end{tabular}
\caption{Test accuracy of depth-2 model trained without hints and with/without tempered softmax.\label{tbl:beta-softmax-sorting-1}}
}
\hfill
\parbox{.42\linewidth}{
\centering
\begin{tabular}{ccc}
\hline
 &Standard&Tempered\\ & softmax &softmax\\
\hline
 Test Acc.\@ 50 & $99.2$ &  $\mathbf{99.7}$\\  
 \hline
 Test Acc.\@ 100 & $52.4$ & $\mathbf{64.8}$  \\ \hline
\end{tabular}
\caption{Test accuracy of depth-2 model trained with successor hints and with/without tempered softmax.\label{tbl:beta-softmax-sorting-2}}
}
\end{table}

As we observe in
\Cref{tbl:beta-softmax-sorting-1,tbl:beta-softmax-sorting-2}, the
introduction of the tempered softmax significantly improves length
generalization of models trained via standard training, as well as
those trained via task hinting. Furthermore, the tempered softmax
helps across all data scale ranges. In particular, even for the model
trained without hints on a training set of size 1M, the test accuracy
on sequences of length 100 increases from $0\%$ to $42\%$ due to the
introduction of the tempered softmax!

\section{Task Hinting for Other Problems}
\label{sec:others}

In this section we discuss the effectiveness of our proposed approach
for another problem: that of incrementing a positive integer, i.e., adding $1$ to it. 
As
we will see, it is quite challenging for transformers to be able to
generalize on unseen lengths even for this simple setting. We again
train decoder-only models that produce one token at a time. Similar to
the case of sorting, we use the $\bot$ token to denote the end of the
input sequence. Each example in the training set is a sequence of the
form: $[1, 2, 3, \bot, 4, 2, 1]$, where the output is being produced
in \emph{reverse order}, given that is the way in which humans tend to
solve this task. The training set contains $1M$ instances of lengths
up to $10$. Similar to the case of sorting, we skew the distribution
towards shorter sequences by sampling $80\%$ of the instances from
lengths up to $4$. Finally, we ensure that $10\%$ of the samples end
with a random sequence of $9$s, since these instances are important
for the model to learn the notion of a {\em carry}.

Solving the increment task via a causal decoder-only network presents
a different set of challenges than sorting---the instance is no longer
permutation-invariant, and as the number of output tokens increases,
the model has to attend to a specific position farther to the left in
the input sequence. We compare the length-generalization properties of
models obtained via standard training versus those obtained via either
task hinting or via introducing the tempered softmax operation. For
task hinting, we consider the natural hint of making the model output
the {\em carry} sequence along with the output sequence. Hence an
instance from the auxiliary task will be structured as
\[ [1,2,3,\bot,4,\uparrow,0,2,\uparrow,0,1,\uparrow,0], \] where the
$\uparrow$ token represents the fact that the model should output the
correct carry value at the next step. We train depth-four transformer
models for this task and evaluate their test accuracy on the task of
solving the increment problem correctly.

\begin{table}[h]
\centering
\begin{tabular}{|c|c|c|c|c|c|c|c|c|c|c|}
\hline
$\downarrow$ Model, $\rightarrow$ n&$11$&$12$&$13$&$14$&$15$&$16$&$17$&$18$&$19$&$20$\\ \hline
Standard&$98.2$& $93.8$& $81.5$& $60.1$ &$41$& $23.2$&$10$& $4.1$&$1.4$&$0.3$\\ \hline
Hinting&$99.4$&$96.4$&$88.4$&$69.2$&$47.7$&$27.1$&$13.7$&$6$&$2.2$&$0.5$\\ \hline
Temp. softmax&$99.8$&$97.5$&$91.4$&$78.3$&$62.1$&$46.4$&$29$&$16$&$8$&$4$\\
\hline
\end{tabular}
\caption{Test accuracy comparison of various models on the increment task.\label{tbl:addition}}
\end{table}

\Cref{tbl:addition} compares the performance of the model trained via
standard training to (a)~the model trained via task hinting, and
(b)~the model trained using the tempered softmax. We observe that
while task hinting helps improve length generalization, the
improvements are smaller compared to the improvements for
sorting. However, we observe that the model based on tempered softmax
helps improve the length generalization to a much greater extent.

\section{Discussion and Limitations}
\label{sec:conclusions}

In this work we proposed {\em task hinting} as an effective approach
for the problem of length generalization. We observe that using hints
that have a strong alignment with the internal biases of the learning
mechanism can result in significant gains in out-of-distribution
robustness for the problem of sorting integers. For this setting, we
use probing and visualization-based techniques to investigate the
internal learning mechanisms; these allow us to explain the success of
the successor-based hints that we use in our experiments. In general,
even these probing/visualization approaches may not always be feasible
for large-scale settings, so designing the appropriate hinting tasks
may be a problem in itself: it would be good to develop a principled
approach for deciding on hinting tasks. 

While we observed that other natural hinting tasks, such as the count
task and the fill task did not help (and sometimes even hurt the
performance), we feel that these are useful auxiliary capabilities for
a sorting network, and it would be good to understand their lack of
success at a deeper level. Moreover, it would also be interesting to
combine multiple hints, and make the network benefit from learn more
than two tasks simultaneously. We tried this approach for the sorting
problem, where we trained the model to do well on all the three types
of hinting tasks simultaneously, but observed mixed or even negative
results.

Our work also proposes the introduction of
length-dependent parameters into the attention mechanism, and observe
that they significantly boost the robustness of the models for both
the sorting problem and the increment problem. It would be interesting
to apply this to larger-scale settings of training language models,
and to evaluate whether any gains in robustness can be obtained on
more general reasoning tasks. Finally, when using the framework of
multitask learning to make the network learn both tasks
simultaneously, we did not make efforts to optimize the various
parameters of the setup, and followed a simple recipe of alternating
gradient updates on each task. Further optimizations in this stage
could lead to better performance.

\section{Acknowledgements}
We thank Nishanth Dikkala and Guru Guruganesh for several useful discussions and insightful comments that helped improve the results and the readability of the paper.

\bibliographystyle{unsrtnat}
\bibliography{references}

\begin{thebibliography}{28}
\providecommand{\natexlab}[1]{#1}
\providecommand{\url}[1]{\texttt{#1}}
\expandafter\ifx\csname urlstyle\endcsname\relax
  \providecommand{\doi}[1]{doi: #1}\else
  \providecommand{\doi}{doi: \begingroup \urlstyle{rm}\Url}\fi

\bibitem[Abbe et~al.(2023)Abbe, Bengio, Lotfi, and
  Rizk]{abbe2023generalization}
Emmanuel Abbe, Samy Bengio, Aryo Lotfi, and Kevin Rizk.
\newblock Generalization on the unseen, logic reasoning and degree curriculum.
\newblock \emph{arXiv preprint arXiv:2301.13105}, 2023.

\bibitem[Anil et~al.(2022)Anil, Wu, Andreassen, Lewkowycz, Misra, Ramasesh,
  Slone, Gur-Ari, Dyer, and Neyshabur]{anil2022exploring}
Cem Anil, Yuhuai Wu, Anders Andreassen, Aitor Lewkowycz, Vedant Misra, Vinay
  Ramasesh, Ambrose Slone, Guy Gur-Ari, Ethan Dyer, and Behnam Neyshabur.
\newblock Exploring length generalization in large language models.
\newblock \emph{Advances in Neural Information Processing Systems},
  35:\penalty0 38546--38556, 2022.

\bibitem[Brown et~al.(2020)Brown, Mann, Ryder, Subbiah, Kaplan, Dhariwal,
  Neelakantan, Shyam, Sastry, Askell, et~al.]{brown2020language}
Tom Brown, Benjamin Mann, Nick Ryder, Melanie Subbiah, Jared~D Kaplan, Prafulla
  Dhariwal, Arvind Neelakantan, Pranav Shyam, Girish Sastry, Amanda Askell,
  et~al.
\newblock Language models are few-shot learners.
\newblock \emph{Advances in neural information processing systems},
  33:\penalty0 1877--1901, 2020.

\bibitem[Chandra et~al.(1984)Chandra, Stockmeyer, and Vishkin]{ChandraSV84}
Ashok~K. Chandra, Larry~J. Stockmeyer, and Uzi Vishkin.
\newblock Constant depth reducibility.
\newblock \emph{{SIAM} J. Comput.}, 13\penalty0 (2):\penalty0 423--439, 1984.
\newblock \doi{10.1137/0213028}.
\newblock URL \url{https://doi.org/10.1137/0213028}.

\bibitem[Chen et~al.(2022)Chen, Wang, Changpinyo, Piergiovanni, Padlewski,
  Salz, Goodman, Grycner, Mustafa, Beyer, et~al.]{chen2022pali}
Xi~Chen, Xiao Wang, Soravit Changpinyo, AJ~Piergiovanni, Piotr Padlewski,
  Daniel Salz, Sebastian Goodman, Adam Grycner, Basil Mustafa, Lucas Beyer,
  et~al.
\newblock Pali: A jointly-scaled multilingual language-image model.
\newblock \emph{arXiv preprint arXiv:2209.06794}, 2022.

\bibitem[Chowdhery et~al.(2022)Chowdhery, Narang, Devlin, Bosma, Mishra,
  Roberts, Barham, Chung, Sutton, Gehrmann, et~al.]{chowdhery2022palm}
Aakanksha Chowdhery, Sharan Narang, Jacob Devlin, Maarten Bosma, Gaurav Mishra,
  Adam Roberts, Paul Barham, Hyung~Won Chung, Charles Sutton, Sebastian
  Gehrmann, et~al.
\newblock Palm: Scaling language modeling with pathways.
\newblock \emph{arXiv preprint arXiv:2204.02311}, 2022.

\bibitem[Crawshaw(2020)]{crawshaw2020multi}
Michael Crawshaw.
\newblock Multi-task learning with deep neural networks: A survey.
\newblock \emph{arXiv preprint arXiv:2009.09796}, 2020.

\bibitem[Devlin et~al.(2018)Devlin, Chang, Lee, and Toutanova]{devlin2018bert}
Jacob Devlin, Ming-Wei Chang, Kenton Lee, and Kristina Toutanova.
\newblock Bert: Pre-training of deep bidirectional transformers for language
  understanding.
\newblock \emph{arXiv preprint arXiv:1810.04805}, 2018.

\bibitem[Dubois et~al.(2019)Dubois, Dagan, Hupkes, and
  Bruni]{dubois2019location}
Yann Dubois, Gautier Dagan, Dieuwke Hupkes, and Elia Bruni.
\newblock Location attention for extrapolation to longer sequences.
\newblock \emph{arXiv preprint arXiv:1911.03872}, 2019.

\bibitem[Heek et~al.(2023)Heek, Levskaya, Oliver, Ritter, Rondepierre, Steiner,
  and van {Z}ee]{flax2020github}
Jonathan Heek, Anselm Levskaya, Avital Oliver, Marvin Ritter, Bertrand
  Rondepierre, Andreas Steiner, and Marc van {Z}ee.
\newblock {F}lax: A neural network library and ecosystem for {JAX}, 2023.
\newblock URL \url{http://github.com/google/flax}.

\bibitem[Jelassi et~al.(2023)Jelassi, d'Ascoli, Domingo-Enrich, Wu, Li, and
  Charton]{jelassi2023length}
Samy Jelassi, St{\'e}phane d'Ascoli, Carles Domingo-Enrich, Yuhuai Wu, Yuanzhi
  Li, and Fran{\c{c}}ois Charton.
\newblock Length generalization in arithmetic transformers.
\newblock \emph{arXiv preprint arXiv:2306.15400}, 2023.

\bibitem[Kingma \& Ba(2014)Kingma and Ba]{kingma2014adam}
Diederik~P Kingma and Jimmy Ba.
\newblock Adam: A method for stochastic optimization.
\newblock \emph{arXiv preprint arXiv:1412.6980}, 2014.

\bibitem[Lan et~al.(2019)Lan, Chen, Goodman, Gimpel, Sharma, and
  Soricut]{lan2019albert}
Zhenzhong Lan, Mingda Chen, Sebastian Goodman, Kevin Gimpel, Piyush Sharma, and
  Radu Soricut.
\newblock Albert: A lite {BERT} for self-supervised learning of language
  representations.
\newblock \emph{arXiv preprint arXiv:1909.11942}, 2019.

\bibitem[Li et~al.(2022)Li, Hopkins, Bau, Vi{\'e}gas, Pfister, and
  Wattenberg]{li2022emergent}
Kenneth Li, Aspen~K Hopkins, David Bau, Fernanda Vi{\'e}gas, Hanspeter Pfister,
  and Martin Wattenberg.
\newblock Emergent world representations: Exploring a sequence model trained on
  a synthetic task.
\newblock \emph{arXiv preprint arXiv:2210.13382}, 2022.

\bibitem[Liu \& Low(2023)Liu and Low]{liu2023goat}
Tiedong Liu and Bryan Kian~Hsiang Low.
\newblock Goat: Fine-tuned {LLaMA} outperforms {GPT-4} on arithmetic tasks.
\newblock \emph{arXiv preprint arXiv:2305.14201}, 2023.

\bibitem[Loshchilov \& Hutter(2016)Loshchilov and Hutter]{loshchilov2016sgdr}
Ilya Loshchilov and Frank Hutter.
\newblock Sgdr: Stochastic gradient descent with warm restarts.
\newblock \emph{arXiv preprint arXiv:1608.03983}, 2016.

\bibitem[Malach(2023)]{malach2023auto}
Eran Malach.
\newblock Auto-regressive next-token predictors are universal learners.
\newblock \emph{arXiv preprint arXiv:2309.06979}, 2023.

\bibitem[Merrill et~al.(2022)Merrill, Sabharwal, and
  Smith]{merrill2022saturated}
William Merrill, Ashish Sabharwal, and Noah~A Smith.
\newblock Saturated transformers are constant-depth threshold circuits.
\newblock \emph{Transactions of the Association for Computational Linguistics},
  10:\penalty0 843--856, 2022.

\bibitem[Nanda et~al.(2023{\natexlab{a}})Nanda, Chan, Liberum, Smith, and
  Steinhardt]{nanda2023progress}
Neel Nanda, Lawrence Chan, Tom Liberum, Jess Smith, and Jacob Steinhardt.
\newblock Progress measures for grokking via mechanistic interpretability.
\newblock \emph{arXiv preprint arXiv:2301.05217}, 2023{\natexlab{a}}.

\bibitem[Nanda et~al.(2023{\natexlab{b}})Nanda, Lee, and
  Wattenberg]{nanda2023emergent}
Neel Nanda, Andrew Lee, and Martin Wattenberg.
\newblock Emergent linear representations in world models of self-supervised
  sequence models, 2023{\natexlab{b}}.

\bibitem[Newman et~al.(2020)Newman, Hewitt, Liang, and Manning]{newman2020eos}
Benjamin Newman, John Hewitt, Percy Liang, and Christopher~D Manning.
\newblock The {EOS} decision and length extrapolation.
\newblock \emph{arXiv preprint arXiv:2010.07174}, 2020.

\bibitem[Nye et~al.(2021)Nye, Andreassen, Gur-Ari, Michalewski, Austin, Bieber,
  Dohan, Lewkowycz, Bosma, Luan, et~al.]{nye2021show}
Maxwell Nye, Anders~Johan Andreassen, Guy Gur-Ari, Henryk Michalewski, Jacob
  Austin, David Bieber, David Dohan, Aitor Lewkowycz, Maarten Bosma, David
  Luan, et~al.
\newblock Show your work: Scratchpads for intermediate computation with
  language models.
\newblock \emph{arXiv preprint arXiv:2112.00114}, 2021.

\bibitem[Press et~al.(2021)Press, Smith, and Lewis]{press2021train}
Ofir Press, Noah~A Smith, and Mike Lewis.
\newblock Train short, test long: Attention with linear biases enables input
  length extrapolation.
\newblock \emph{arXiv preprint arXiv:2108.12409}, 2021.

\bibitem[Radford et~al.(2019)Radford, Wu, Child, Luan, Amodei, Sutskever,
  et~al.]{radford2019language}
Alec Radford, Jeffrey Wu, Rewon Child, David Luan, Dario Amodei, Ilya
  Sutskever, et~al.
\newblock Language models are unsupervised multitask learners.
\newblock \emph{OpenAI blog}, 1\penalty0 (8):\penalty0 9, 2019.

\bibitem[Touvron et~al.(2023)Touvron, Lavril, Izacard, Martinet, Lachaux,
  Lacroix, Rozi{\`e}re, Goyal, Hambro, Azhar, et~al.]{touvron2023llama}
Hugo Touvron, Thibaut Lavril, Gautier Izacard, Xavier Martinet, Marie-Anne
  Lachaux, Timoth{\'e}e Lacroix, Baptiste Rozi{\`e}re, Naman Goyal, Eric
  Hambro, Faisal Azhar, et~al.
\newblock Llama: Open and efficient foundation language models.
\newblock \emph{arXiv preprint arXiv:2302.13971}, 2023.

\bibitem[Tu et~al.(2023)Tu, Azizi, Driess, Schaekermann, Amin, Chang, Carroll,
  Lau, Tanno, Ktena, et~al.]{tu2023towards}
Tao Tu, Shekoofeh Azizi, Danny Driess, Mike Schaekermann, Mohamed Amin,
  Pi-Chuan Chang, Andrew Carroll, Chuck Lau, Ryutaro Tanno, Ira Ktena, et~al.
\newblock Towards generalist biomedical {AI}.
\newblock \emph{arXiv preprint arXiv:2307.14334}, 2023.

\bibitem[Wei et~al.(2022)Wei, Wang, Schuurmans, Bosma, Xia, Chi, Le, Zhou,
  et~al.]{wei2022chain}
Jason Wei, Xuezhi Wang, Dale Schuurmans, Maarten Bosma, Fei Xia, Ed~Chi, Quoc~V
  Le, Denny Zhou, et~al.
\newblock Chain-of-thought prompting elicits reasoning in large language
  models.
\newblock \emph{Advances in Neural Information Processing Systems},
  35:\penalty0 24824--24837, 2022.

\bibitem[Zhang et~al.(2022)Zhang, Backurs, Bubeck, Eldan, Gunasekar, and
  Wagner]{zhang2022unveiling}
Yi~Zhang, Arturs Backurs, S{\'e}bastien Bubeck, Ronen Eldan, Suriya Gunasekar,
  and Tal Wagner.
\newblock Unveiling transformers with {LEGO}: a synthetic reasoning task.
\newblock \emph{arXiv preprint arXiv:2206.04301}, 2022.

\end{thebibliography}

\appendix
\newpage

\section{Formal Constructions [Proof of Theorem~\ref{thm:main}]}

We now show how to implement a min/successor operation via two-layer transformers, which allows us to sort using next-token predictions.

\subsection{Notation}

Fix an alphabet $\Sigma$ of size $q$. We have one special symbol $\bot$,
which is the end-of-sequence delimiter. Let $\Sigma'$ denote the
extended alphabet $\Sigma \cup \{\bot\}$. We associate $\Sigma$ with
the naturals $\{1, 2, \ldots, q\}$, with the usual total order on them.

Since we seek to sort sequences using next-token prediction, the input
is a sequence of length $T$ consists of three conceptual parts:
\begin{enumerate}
\item the pre-delimiter part: a sequence of $n$ symbols
  $\s_0, \s_1, \ldots, \s_{n-1}$ where each $\s_i \in \Sigma$. These
  represent the unsorted input. 
\item the end-of-sequence delimiter: $\s_{n} = \bot$.
\item the post-delimiter part: a sequence of some $i = T-n-1$ symbols
  $\s_{n+1}, \s_{T+2}, \ldots, \s_{T-1}$ again from $\Sigma$, which
  ideally represent the smallest $i$ symbols in the input
  $\bs_{[0:n-1]}$ (in non-decreasing order).
\end{enumerate}
Given this sequence $\bs$ we want to predict the next symbol in the
sorted order of the input $\bs_{[0:n-1]}$.

\subsection{The Transformer Architecture}

The process works as follows: 
\begin{enumerate}
\item The initial \emph{embedding function} $\Enc: \Sigma' \to \RR^d$ maps
  each symbol to a vector in $\RR^d$. Let $\bX_0 \in \RR^{T \times d}$
  be the embedding of the input sequence $\bs$, where the $i^{th}$ row
  of $\bX_0$ equals $\Enc(\sigma_i)$.  We use $\bX_{ti}$ to denote the
  $i^{th}$ row of $\bX_t$.

\item There are $b$ \emph{attention blocks} which transform this input:
  we denote the
  operation of attention block $t$ by
  $B_t: \RR^{T\times d} \to \RR^{T\times d}$, and hence
  \[ \bX_t := B_t(\bX_{t-1}). \] Each block contains a \emph{self-attention
    layer} and \emph{a multi-layer perceptron (MLP) layer}, followed a
  \emph{layer normalization} operation, such that 
  \[ B_t := \bigg( \LN \circ (I + f^{mlp}_t) \circ (I + f^{attn}_t)
    \bigg). \] The identity maps are the \emph{residual} stream to
  which the results of the various operations get repeatedly added
  in.

\item Each \emph{self-attention layer} consists of $h$
  \emph{attention heads}: the columns of the matrix $\bX_t$ are split
  into $h$ matrices $\bX_{t1}, \ldots, \bX_{th}$, each with $d/h$
  columns and $T$ rows. Each head  is specified by matrices
  $Q,K,V$. It takes a matrix $\overline{X} \in \RR^{T \times
    d/h}$ and produces a matrix of the same dimensions as follows:
  \[ f^{attn}(\overline{\bX}; K,Q,V) := \text{smax}_\tau(\overline{\bX} KQ^\intercal
    \overline{\bX}^\intercal) \overline{\bX} V \] Here the $\text{smax}$ operator takes a matrix
  $A$ and a parameter $\tau$ and defines
  \[ \text{smax}_\tau(A)_{ij} = \frac{ e^{\tau A_{ij}} \, \mathbf{1}_{(i \geq
        j)}}{\sum_{j' \leq j} e^{\tau A_{ij'}}}. \] 
  (Note that the above operator combines the softmax operation and the auto-regressive behavior.) Finally, the resulting
  $h$ sub-matrices are concatenated together to give the result of the
  entire self-attention layer; let $f^{attn}_t$ denote the composite
  function. Define $\bY_t := (I + f^{attn}_t) \bX_{t-1}$ as the result
  of adding back the original signal to the result.

\item Next, the \emph{multi-layer perceptron} (MLP) layer (which in our case
  is a two-layer perceptron) is a transformation
  $f^{mlp}(\bY; W_1, W_2, b_1, b_2)$ that is specified by two matrices
  $W_1, W_2$ and bias vectors $b_1, b_2$. It is the result of applying
  the following map to each row $y^\intercal$ of $\bY$ separately:
  \[ y \mapsto W_2 \, \sigma( W_1y + b_1) + b_2. \] Here the map
  $\sigma(\cdot)$ is usually the component-wise ReLU operation (or in
  more complicated settings, other non-linear operators like GeLU or
  GLU).
  
\item The final piece in each block is the \emph{layer normalization}
  operation, which again is applied to each row of the current
  embedding independently. Given a vector $x \in \RR^d$, it subtracts
  $\mu := \|x\|_1/d$ from each coordinate to make it zero-mean, and
  then divides each entry by $\sigma := \sqrt{\sum_i x_i^2/d}$; this
  makes the Euclidean length $\sqrt{d}$. We denote this operation by
  $\LN$.

\item Unrolling, the entire transformer map is
  \[ \bX_b = \big( B_b \circ B_{b-1} \circ \cdots \circ B_1\big) \, \bX_0. \] 
  
\item The final transformation is the decoding/unembedding operation,
  which takes $\bX_b \in \RR^{T\times d}$ and applies some decoding
  map $\Dec: \RR^d \to \Sigma'$ independently on each row of
  $\bX_b$. This produces characters in $\Sigma'$---these are the
  \emph{predictions} for the next symbols. For our decoder-only
  constructions, the only relevant prediction is that of the last
  symbol: we output this prediction $\Dec(\bX_{b,T-1})$ as the next
  token, thereby increasing the length by $1$---this longer string is
  then the input for the next iteration.
\end{enumerate}

We now show how to implement each of these attention blocks for the
sorting network.

\subsection{The Encoding Function}

Fix a set of unit vectors $\{\be_s, \be_s'\}_{s \in \Sigma'}$ which
are all orthogonal to each other in $\RR^d$. The initial embedding is
simple: each symbol $a \in \Sigma$ is encoded by the vector
$\be_a + \be'_a$, and the end-of-input delimiter is encoded as
$\be_\bot + \be_\bot'$. This gives us the input embedding $\bX_0$.

\subsection{Block \#1}

The first block has two goals: (i)~it gets each token to implement a ``min/copy'' operation (in which the end-of-input delimiter predicts the minimum element from the input, whereas each other token just predicts itself), and (ii)~the tokens corresponding to the same symbol before and after the end-of-input delimiter distinguish themselves, so that the second block can act on them accordingly.

\subsubsection{Block \#1: Self-Attention Layer}

There are two attention heads in the first self-attention layer, each
getting some $d/2$ columns of the matrix $\bX_0$. We denote the
resulting two sub-matrices by  
$\bX_{01}, \bX_{02} \in \RR^{T \times d/2}$, and ensure that
for each $s$, the span of $\{\be_s\}_{s \in \Sigma'}$ lies in the subspace corresponding to the
first $d/2$ coordinates, and the span of $\{\be_s'\}$ lies in the one
for the other $d/2$ coordinates.

In the entire construction, we set $\tau = 3 \ln n$, where $n$ is the
length of the input. Define the $Q, K, V$ matrices for the attention
heads as follows:
\begin{itemize}
\item Attention Head \#1, which operates on a subspace containing the
  vectors $\{ \be_0, \be_1, \be_2, \ldots, \be_q\}$: for some
  positive scalar $C \geq 1$ to be specified below, define
  \begin{alignat}{3}
    Q\be_a &= \be_a + C\be_\bot & \qquad \qquad K\be_a &= \be_a &\qquad
    \qquad V\be_a &= \te_a \\
    Q\be_\bot &= \be_\bot & \qquad \qquad K\be_\bot &= \be_\bot & \qquad \qquad
    V\be_\bot &= \te_\bot.
  \end{alignat}
  (Here, and subsequently, the matrices $Q,K,V$ map all vectors
  orthogonal to the specified vectors to zero.) The vectors
  $\{\te_s\}_{s \in \Sigma'}$ are fresh orthonormal vectors.

\item Attention Head \#2, which  operates on a subspace containing the
  vectors $\{ \be_0', \be_1', \be_2', \ldots, \be_q'\}$: define
  \begin{alignat}{3}
    Q\be_a' &= \be_a' & \qquad \qquad K\be_a' &= \be_a' &\qquad
    \qquad V\be_a' &= \he_a' \\
    Q\be_\bot' &= \sum_{b \in \Sigma} \gamma_b \, \be_b' & \qquad
    \qquad K\be_\bot' &= \be_\bot' & \qquad \qquad V\be_\bot &= 0.
  \end{alignat}
  Here $\{\gamma_b\}_{b \in \Sigma}$ are also values to be specified
  soon. Again, the vectors $\{\he_b'\}_{b \in \Sigma}$ are fresh
  orthonormal vectors.
\end{itemize}
This means that for any symbol $a \in \Sigma$ at some position $i$
before the $\bot$ delimiter, the first attention head outputs
\[ \frac{\sum_{j \leq i: \sigma_j = a} e^\tau \; \te_{\sigma_j} + \sum_{j
      \leq i: \sigma_j \neq a} \te_{\sigma_j}}{\sum_{j \leq i:
      \sigma_j = a} e^\tau + \sum_{j \leq i : \sigma_j \neq a} 1} =
  \frac{n_{a,[0,i]} \, e^\tau \, \te_a + \sum_{j \leq i: \sigma_j \neq a}
    \te_{\sigma_j}}{n_{a, [0,i]} e^\tau + (i+1-n_{a,[0,i]})} \] Here
$n_{a,[x,y]}$ is the number of occurrences of $a$ in the multiset
$\{\sigma_x, \ldots, \sigma_y\}$.  Now since $\tau \geq 3\ln n$, most of
the attention is on all occurrences of the same symbol $a$ seen thus
far, and hence this expression is
\[ (1-O(\nicefrac{1}{n^4})) \cdot \te_a + O(\nicefrac{1}{n^2}) \cdot \bu_{1i},  \]
where $\bu_{1}$ is some ``error'' vector of unit norm. Similarly, the second
attention head gives 
\[ (1-O(\nicefrac{1}{n^4})) \cdot \he_a + O(\nicefrac{1}{n^2}) \cdot
  \bu_{2i}, \] for some other error vector $u_2$. Hence, adding back in
the residual, we get that the $i^{th}$ entry (for $i < n$, where
$\sigma_i = a$ for some $a \in \Sigma$) gives us
\begin{gather}
  (I + f^{attn}_1)(\bX_{0i}) \approx \bX_{0i} + \te_a + \he_a = \be_a
  + \be_a' + \te_a + \he_a.
\end{gather}
Here and henceforth, we will use the ``$\approx$'' to hide error vectors of length $O(1/n^2)$.

Now a similar analysis shows that for position $i > n$ (such that
$\bX_{0i} = a$), setting $C = 3$ then most of $a$'s attention
(in the first head) is on the $\bot$ delimiter, and hence
\begin{gather}
  (I + f^{attn}_1)(\bX_{0i}) \approx (\be_a + \be_a') + \te_\bot +
  \he_a'. \label{eq:4}
\end{gather}
Finally, the $\bot$ delimiter pays most of its attention to itself in
first head, whereas in the second head it pays attention to all the
tokens (weighted by the $e^{\tau \gamma_b}$ multipliers). Defining
$\alpha_b := e^{\tau \gamma_b}$, we get
\begin{align}
  (I + f^{attn}_1)(\bX_{0i}) \approx (\be_\bot +
  \be_\bot') +  \te_\bot + \frac{\sum_{b}
  \alpha_b n_{b,[0,n)} \, \he_b'}{\sum_{b} \alpha_b
  n_{b,[0,n)}+1}. \label{eq:3}
\end{align}
Since we have identified the symbols of $\Sigma$ with
$\{1,2,\ldots,q\}$, we can define $\gamma_b = (q-b+1)$, and hence $\ln \alpha_b := 3(q-b+1)\ln n$.
Since the $\alpha_b$ values decrease rapidly as $b$ increases, the
fraction on the right assigns most of its weight to vector $\he_b'$
corresponding to the minimum element in the input $\sigma_{[0,n)}$. In
other words, we get
\begin{align}
  (I + f^{attn}_1)(\bX_{0i}) \approx (\be_\bot +
  \be_\bot') +  \te_\bot + \min_{b \in \bs_{[0,n-1]}} \he_b'.
\end{align}

Let us denote the output of the first self-attention layer by $\bY_1$; i.e.,
\[ \bY_1 := (I + f^{attn}_1)(\bX_{0}). \]

\subsubsection{Block \#1: MLP Layer}

Recall that the MLP layer is applied to each embedding separately, and
there is no interaction between the embeddings of different tokens.
The first MLP layer has two goals:
\begin{enumerate}
\item The first goal is to convert the $\te_\bot$ vector in embedding
  of some post-delimiter $a$ to the corresponding $-\te_a$.  To this
  end, the
  \begin{gather}
    f^{mlp}_{1,1}(x) := \sum_{b \in \Sigma} \sigma( \ip{x, \be_b} +
    \ip{x, \te_\bot} - 1) (- \te_b' - \te_\bot).
  \end{gather}
  Recall that $\sigma(z) := \max(0,z)$ is the ReLU function.
 
\item The second goal is to shift the coordinates of the $\te_a$
  vectors, so that they appear in the second half of the coordinates
  instead of the first. For this we use
  \begin{gather}
    f^{mlp}_{1,2}(x) := \sum_{b \in \Sigma} \bigg((\sigma(\ip{x,\te_b}
    - \sigma(\ip{x,-\te_b}) \cdot (\te_b' - \te_b) \bigg)
  \end{gather}
\end{enumerate}
Finally, $f^{mlp}_1(x) := f^{mlp}_{1,1}(x) + f^{mlp}_{1,2}(x)$. This
gives us the output of the first attention block:
\[ \bX_1 := (I + f^{mlp}_1)(\bY_{1}). \] As mentioned above, we do not use
the layer normalization in this construction, so this $\bX_1$ is now fed
to the second attention block.

To summarize,
\begin{gather*}
  \bX_{1i} = (I + f^{mlp}_1)(\bY_{1i}) = \begin{cases}
    (\be_a + \be_a') + \te_a' + \he_a' + \bu_{3i} \qquad  \qquad &
    \text{for $i<n$} \\
    (\be_a + \be_a') - \te_a' + \he_a' + \bu_{3i} & \text{for $i > n$, and} \\
    (\be_\bot + \be_\bot') + \te_\bot + \min_{b \in \bs_{[0,n-1]}} \he_b' +
    \bu_{3i} &
    \text{for $i=n$.}  
  \end{cases} 
\end{gather*}
The error vectors $\bu_{3i}$ above have magnitude
$O(\nf1{n^2})$.

\subsection{Block \#2}

The second block now ensures that the $\bot$ token predicts the minimum element, whereas each other token predicts its successor. The non-trivial part of this construction arises from duplicates in the input, so that each symbol $a \in \Sigma$ has to infer whether the number of copies of $a$ already output equals the number in the input part of $\bs$, and accordingly predict whether to output another $a$ or the successor to $a$. (Observe that this is an \emph{ordinal} concept, and not a \emph{cardinal} one: the network does not need the actual count of the $a$'s that have been output, but to just know whether the number of $a$'s output is strictly less than the number in the input.)

\subsubsection{Block \#2: Self-Attention Layer}

The self-attention layer of the second block again has two attention heads: 
\begin{itemize}
\item Attention Head \#1, which again operates on a subspace
  containing the ``unprimed'' vectors: 
  \begin{alignat}{3}
    Q\be_a &= \be_a & \qquad \qquad K\be_a &= \be_a &\qquad
    \qquad V\te_a &= \he_a \\
    Q\be_\bot &= \be_\bot & \qquad \qquad K\be_\bot &= \be_\bot & \qquad \qquad
    V\be_\bot &= 0.
  \end{alignat}
  Again, the matrices $Q,K,V$ map all vectors orthogonal to the
  specified vectors to zero. Recall that the $\tau$ parameter is the
  softmax operator is set to $3 \ln n$.

\item Attention Head \#2, which operates on the primed vectors: define
  \begin{alignat}{3}
    Q\be_a' &= \sum_{b > a} \gamma_b \be_b' & \qquad \qquad K\be_a' &= \be_a'  &\qquad
    \qquad V\be_a' &= \eps \he_a'  \\
    Q\be_\bot' &= \be_\bot' &  \qquad \qquad K\be_\bot'&= \be_\bot'
    & V\be_\bot'&= 0.
  \end{alignat}
\end{itemize}
(We will fix the value of $\eps > 0$ below.)

Since we are at the final block, we are no longer concerned with the
part of the input in $\s_{[0:n-1]}$, and hence focus on positions $n$
and beyond. The $\bot$ delimiter at position $n$ primarily pays
attention to itself in both attention heads, since $\tau$ is $\Omega(\log
n)$. This means it remains unchanged, and
\begin{gather}
  (I + f^{attn}_2)(\bX_{1n}) = (\be_\bot + \be_\bot') +
  \te_\bot + \min_{b \in \bs_{[0,n-1]}} \he_b' + \bu_{4n}, 
\end{gather}
where the new error vector $\bu_{4n}$ is still of the order $O(1/n^2)$.

Next, consider any position $i > n$, such that $\sigma_i = a$ for some
$a \in \Sigma$. The first attention head gives
\begin{gather}
  \frac{\sum_{j \leq i: \sigma_j = a} e^C \; V(\bX_{1j}) + \sum_{j \leq
      i: \sigma_j \neq a} V(\bX_{1j})}{e^C \, n_{a,[0.i]} +
    (n-n_{a,[0,i]})} = 
  \frac{e^C (n_{a,[0,n]} - n_{a,[n+1,i]}) \,\he_a + \sum_{j \leq
      i: \sigma_j \neq a} \he_{\sigma_j}}{e^C \, n_{a,[0.i]} +
    (n-n_{a,[0,i]})} 
\end{gather}
Again, since $C = \Omega(\ln n)$, this is approximately
\begin{gather}
  (n_{a,[0,n]} - n_{a,[n+1,i]})\, \he_a + \bu_{4i}, \label{eq:1}
\end{gather}
where the error vector $\bu_{4i}$ has tiny norm $O(1/n^2)$. The second
attention head for the same symbol $\s_i$ gives
\begin{gather}
  \frac{\sum_{b > a} \alpha_b n_{b,[0,i]} \, \eps \,\he_b' + \sum_{b \leq a} n_{b,[0,i]}\, \eps  \,\he_b'}{\sum_{b > a} \alpha_b n_{b,[0,i]} + \sum_{b \leq a} n_{b,[0,i]}}.
\end{gather}
Recall that $\alpha_b = e^{\tau \gamma_b} = n^{3 (q-b+1)}$. If the
symbol $a$ is not the largest symbol of the input (so that other
symbols $b > a$ follow it in the input), this expression is
$\eps(\min_{b > a} \he_b' + \bu_{4i}')$, with the error vector
$\bu_{4i}'$ having a tiny norm compared to $\min_{b > a} \he_b'$. As
before, we define
\[ \bY_2 := (I + f^{attn}_2)(\bX_1) \]
to be the outcome of this self-attention layer.

Let $\widehat{P}$ be the projection of these embeddings on the subspace
spanned by the ``hatted'' vectors $\{\he_a, \he_a'\}_{a \in
  \Sigma}$. Then
\begin{align}
  \widehat{P} \bY_{2i} &= \widehat{P} (I + f^{attn}_2)(\bY_{1i})
                         \notag \\ &= \begin{cases}
    \min_{b \in \bs_{[0,n-1]}} \he_b' +
    \bu_{5i} &     \text{for $i=n$, and}  \\
    (n_{a,[0,n]} - n_{a,[n+1,i]})\, \he_a + \he_a' +   \eps \frac{\sum_{b > a}
      \alpha_b n_{b,[0,i]} \,\he_b' + \sum_{b \leq a}
      n_{b,[0,i]}\,\he_b'}{\sum_{b > a} \alpha_b n_{b,[0,i]}
      + \sum_{b \leq a} n_{b,[0,i]}} + \bu_{5i} & \text{for $i > n$.}
  \end{cases} \label{eq:5}
\end{align}

\subsubsection{Block \#2: MLP Layer}

The final MLP layer of the second and final block has a simple task:
\begin{gather}
  f^{mlp}_2(x) := \sum_{b \in \Sigma} \sigma( \ip{x, -\te_b}) \cdot
  (- \,\he_b').
\end{gather}
This has the effect of adding in $(-\he_a')$ to any post-delimiter
$a$, and hence``nullifying'' the $\he_a'$. The net effect (again seen
after projection onto the hatted subspace) is
\begin{align*}
  \widehat{P} \bX_{2i} &= \widehat{P} (I + f^{mlp}_2)(\bY_{2i}) \\ &= \begin{cases}
    \min_{b \in \bs_{[0,n-1]}} \he_b' +
    \bu_{5i} &     \text{for $i=n$, and}  \\
    (n_{a,[0,n]} - n_{a,[n+1,i]})\, \he_a + \eps \frac{\sum_{b > a}
      \alpha_b n_{b,[0,i]} \,\he_b' + \sum_{b \leq a}
      n_{b,[0,i]}\,\he_b'}{\sum_{b > a} \alpha_b n_{b,[0,i]}
      + \sum_{b \leq a} n_{b,[0,i]}} + \bu_{5i} & \text{for $i > n$.}
  \end{cases} 
\end{align*}

\subsection{The Decoding Layer}

\begin{proof}[Proof of Theorem~\ref{thm:main}]

The decoding (or unembedding) layer outputs the element $a$
for which the vector $\he_a + \he_a'$ has the largest inner product
with the current embedding. In other words, $\sigma_i$ predicts
\begin{gather}
  \arg\max_{a \in \Sigma} ~~~\ip{ \bX_{2i}, \he_a + \he_a'}. \label{eq:2}
\end{gather}
From the above construction we have the following properties that establish the correctness of the network:
\begin{enumerate}
\item For the delimiter at position $i = n$, this is simply the
  minimum element from $\bs_{[0,n-1]}$.
\item For any other location $i > n$ with $\sigma_i = a$, there are two
  cases:
  \begin{enumerate}
  \item Suppose there are multiple copies of $a$ in the input $\bs_{[0,n-1]}$, and not
    all of them have been output yet. This means
    $n_{a,[0,n]} > n_{a,[n+1,i]}$, and hence the maximum
    in~(\ref{eq:2}) is achieved by $a$ itself, as long as
    $\eps \leq \nf12$, say. This results in predicting and outputting
    another copy of $a$.
  \item Else suppose the number of copies of $a$ in the output already
    equals that in the input. In this case, the argmax in~(\ref{eq:2})
    is achieved at the smallest element $b \in \bs_{[0,n-1]}$ that is
    larger than $a$; this is indeed the correct ``successor'' element
    for $a$ to predict. One exception is when $i = 2n$, but then we do not need
    any further predictions.
  \end{enumerate}
\end{enumerate}
Here we have crucially used that the
maximizing vector has norm at least a constant, which means that the
error vectors of length $O(1/n^2)$ do not alter the result. 
\end{proof}

\subsection{The Layer Normalization}
\label{sec:layer-normalization}

The construction above (using $d = O(|\Sigma|)$ coordinates) did not
use the layer normalization operation; however, we can convert it to
incorporate this operation as well. Recall that layer normalization
operates on the embedding of each token independently: (i) given a
vector $x \in \RR^d$, it subtracts the mean $\mu_x := \frac1d \|x\|_1$
from each coordinate, and then (ii)~renormalizes it to have squared
norm $d$.

We take the above construction using vectors $x \in \RR^d$ and extend
it by adding $d$ new coordinates and appending an analogous
construction using the negative of these vectors. The new embedding
$\bar{x}$ has mean $\mu_{\bar{x}} = 0$, and hence the step~(i) of
layer normalization does not change anything.

This means that at the end of the first block, each of the embeddings
in our construction have squared length $\approx 4$. This means the
renormalization only changes the magnitude of the embeddings, but
their relative sizes remain the same. Consequently, the computations
in the second block remain unchanged. The final layer normalization
again shifts and renormalizes the embedding, but this does not change
the outcome.

\subsection{Agreement with Experimental Results}

\begin{enumerate}
\item Consider the embedding after the first self-attention layer: decoding
  this embedding of $\bot$ (given in (\ref{eq:3}) gives us the minimum
  element, whereas decoding the embedding of any $\sigma_i$ for
  $i > n$ (as given in (\ref{eq:4})) gives us the element $\sigma_i$
  itself. This ``min/copy'' behavior can be observed in the
  experimental results.

\item After the second self-attention layer, consider the last occurrence
  of any symbol $a$ in the output (say at some position $i > n$, as
  given in (\ref{eq:5})): since $n_{a,[0,n]} = n_{a,[n+1,i]}$ by our
  assumption, decoding this embedding puts most of its mass along $a$
  (due to $\he_a'$) and its successor (due to
  $\eps \frac{\sum_{b > a} \alpha_b n_{b,[0,i]} \,\he_b' + \sum_{b
      \leq a} n_{b,[0,i]}\,\he_b'}{\sum_{b > a} \alpha_b n_{b,[0,i]} +
    \sum_{b \leq a} n_{b,[0,i]}}$). Again, this ``Identity+Successor''
  behavior shows up in the experiments.

\item Finally, the last MLP layer nullifies the mass on the $a$ token
  itself, thereby leaving most of the mass on the successor. This
  aspect also shows up in the experiments.
\end{enumerate}

\section{Experimental Settings and Hyperparameters}
\label{sec:hparam}

\noindent \textbf{Sorting Dataset construction.}
For the sorting problem, we construct the data set for the main task
by sampling a length in $\{2,3, \ldots, 20\}$ from a skewed
distribution. For a given length, we sample an input sequence by
independently drawing random numbers in $\{1,100\}$ uniformly at
random with replacement, for each input position. The skewed length
distribution places $80\%$ of the probability mass equally on lengths
in $\{2,3,4,5\}$ and the remaining $20\%$ uniformly on lengths in
$\{6, 7, \ldots, 20\}$.

We also train models on variants of the above dataset that contain
non-trivial repetitions. These datasets are created by first picking a
length $\ell$ from the same skewed distribution. With probability
$0.9$, we follow the same procedure as above for creating the input
sequence. With the remaining probability $0.1$, we pick $\frac \ell 2$
numbers uniformly at random from $\{1,2,\ldots, 100\}$ (without
replacement) and create the input sequence by independently drawing
numbers from this set (uniformly, with repetitions) for each position
in the input sequence.

The training dataset for the \emph{successor hint} is created in the
same manner as above: having picked the input sequence, we pick a
position uniformly at random and use the corresponding number to
create the successor hint. For creating the \emph{count hint} dataset
we again pick the length of the sequence from the skewed
distribution. Then we pick two numbers $a,b$ at random (without
replacement) from $\{1, 2, \ldots, 100\}$. Finally, with probability
$0.5$ we repeat them the same number of times, and with the remaining
probability either $a$ or $b$ is chosen (equally) to be the
under-represented number. The amount of under-representation is chosen
uniformly among the valid choices, but restricted to be at most
$5$. Similarly, for the \emph{fill task}, we choose the length $\ell$
from the skewed distribution and choose to repeat the element some
number of times uniformly chosen from
$\{1, 2, \ldots, \lfloor \ell/2 \rfloor\}$. To construct the test set
for each sequence length, we sample $100k$ examples uniformly at
random with replacement.

All our models for the sorting network use the architecture and
parameters detailed in \Cref{tbl:sorting-hparam}. We do not use
position embeddings in our architectures, as causal decoder-only
models are not permutation invariant. In all our experiments with
using fixed and learned positional embeddings, we observed comparable
or even worse performance compared to models using no positional
embedding.

\begin{table}[h]
    \centering
    \begin{tabular}{|c|c|}
    \hline
         Parameter & Value \\\hline
         Embedding size $d$ & $1024$\\
         Vocabulary size $q$ & $103$ \\
         Position embedding type & None\\
         \# Attention heads $h$ & $16$\\
         MLP inner dimensionality $d'$ & 2048\\
         Sequence length & $512$\\
         Base learning rate & 1e-5\\
         Optimizer & Adam\\
         LR warmup & Linear for $10$ epochs\\
         LR decay schedule & Cosine, one cycle with default parameters\\
         Dropout & None\\
         Activation & GELU\\
         \hline
    \end{tabular}
    \caption{Hyperparameters for the sorting task.    \label{tbl:sorting-hparam}
}
\end{table}

\noindent \textbf{Increment Dataset Construction.} For the problem of
incrementing a positive integer, we construct the data set for the
main task by sampling a length in $\{2,3, \ldots, 10\}$ from a skewed
distribution. The distribution places $80\%$ mass equally on lengths
$\{2,3,4\}$ and the remaining $20\%$ equally on lengths
$\{5, \ldots, 10\}$. For each length $\ell$, with probability $0.9$ we
sample input position $0$ (the most significant digit) uniformly at
random from $\{1, \ldots, 9\}$ and the remaining positions uniformly
at random (with replacement) from $\{0, \ldots, 9\}$. With the
remaining probability of $0.1$. we randomly replace the last $k$
positions with the digit $9$, where $k$ is chosen uniformly from
$\{1,2,\ldots, \ell\}$. Our test dataset for each length consists of
$100k$ numbers chosen uniformly at random. For the increment task we
use the parameters detailed in \Cref{tbl:increment-hparam} and
always train depth-$4$ models.

\begin{table}[h]
    \centering
    \begin{tabular}{|c|c|}
    \hline
         Parameter & Value \\\hline
         Embedding size $d$ & $1024$\\
         Vocabulary size $q$ & $14$ \\
         Position embedding type & None\\
         \# Attention heads $h$ & $16$\\
         MLP inner dimensionality $d'$ & 2048\\
         Sequence length & $512$\\
         Base learning rate & 1e-5\\
         Optimizer & Adam\\
         LR warmup & Linear for $10$ epochs\\
         LR decay schedule & Cosine, one cycle with default parameters\\
         Dropout & None\\
         Activation & GELU\\
         \hline
    \end{tabular}
    \caption{Hyperparameters for the increment task.    \label{tbl:increment-hparam}}
\end{table}

\end{document}